\documentclass{article}

\usepackage[total={6in,8in}]{geometry}
\usepackage[utf8]{inputenc} 
\usepackage[T1]{fontenc}    
\usepackage{hyperref}       
\usepackage{url}            
\usepackage{booktabs}       
\usepackage{amsmath}
\usepackage{amsthm}
\usepackage{amssymb}
\usepackage{amsfonts}       
\usepackage{nicefrac}       
\usepackage{microtype}      
\usepackage{rotating}
\usepackage{multirow}
\usepackage{color}
\usepackage{times}
\usepackage{natbib}
\graphicspath{ {images/} }

\title{But How Does It Work in Theory? \\ Linear SVM with Random Features}

\newtheorem{theorem}{Theorem}
\newtheorem{lemma}{Lemma}

\newtheorem{definition}{Definition}
\newtheorem{assumption}{Assumption}

\newcommand*{\thmref}[1]{Theorem~\ref{#1}}

\newcommand*{\lemref}[1]{Lemma~\ref{#1}}

\newcommand*{\figureref}[1]{Figure~\ref{#1}}

\renewcommand*{\eqref}[1]{Equation~\ref{#1}}

\author{
  Yitong Sun \\
  Department of Mathematics \\
  University of Michigan \\
  Ann Arbor, MI, 48109\\
  \texttt{syitong@umich.edu} \\
  \and
  Anna Gilbert\\
  Department of Mathematics\\
  University of Michigan\\
  \texttt{annacg@umich.edu} \\
  \and
  Ambuj Tewari \\
  Department of Statistics\\
  University of Michigan\\
  \texttt{tewaria@umich.edu} \\
}

\begin{document}
\maketitle
\newcommand*{\p}{{\operatorname{\mathbb{P}}}}
\newcommand*{\e}{{\operatornamewithlimits{\mathbb{E}}}}
\renewcommand{\d}{{\mathrm{d}}}
\renewcommand{\r}{{\mathbb{R}}}
\newcommand\Norm[1]{\left\Vert #1\right\Vert }
\newcommand{\argmin}{\operatornamewithlimits{argmin}}
\newcommand{\argmax}{\operatornamewithlimits{argmax}}
\newcommand\ip[2]{\langle#1,#2\rangle}
\renewcommand{\t}{\intercal}
\newcommand{\ep}{\varepsilon}
\newcommand\setbar[1]{\left|#1\right.}
\newcommand{\COS}{\operatorname{COS}}

\begin{abstract}
We prove that, under low noise assumptions, the support vector machine with $N\ll m$
random features (RFSVM) can achieve the learning rate faster than $O(1/\sqrt{m})$ on
a training set with $m$ samples when an optimized feature map is used.
Our work extends the previous fast rate analysis
of random features method from least square loss to 0-1 loss.
We also show that the reweighted feature selection method, which
approximates the optimized feature map,
helps improve the performance of RFSVM in experiments on a synthetic data set.
\end{abstract}


\section{Introduction}

Kernel methods such as kernel support vector machines (KSVMs) have been widely and
successfully used in classification tasks (\cite{Steinwart2008}). The power
of kernel methods comes from the fact that they implicitly map the data to a
high dimensional, or even infinite dimensional, feature space, where points
with different labels can be separated by a linear functional.
It is, however, time-consuming to compute the kernel matrix
and thus KSVMs do not scale well to extremely large datasets.
To overcome this challenge, researchers have developed various ways to efficiently
approximate the kernel matrix or the kernel function.

The random features method, proposed by~\cite{Rahimi2008}, maps the data
to a finite dimensional feature space as a random approximation to the feature
space of RBF kernels.  With explicit finite
dimensional feature vectors available, the original KSVM is converted to
a linear support vector machine (LSVM), that can be trained by faster algorithms
(\cite{Shalev-Shwartz2011,Hsieh2008}) and tested in constant time with respect to the number of training
samples. For example, \cite{Huang2014} and \cite{Dai2014} applied RFSVM or its variant to datasets containing
millions of data points and achieved performance comparable to deep neural nets.

Despite solid practical performance, there is a lack of clear theoretical
guarantees for the learning rate of RFSVM.
\cite{Rahimi2009} obtained a risk gap of order $O(1/\sqrt{N})$ between
the best RFSVM and KSVM classifiers, where $N$ is the number of features.
Although the order of the error bound is correct for
general cases, it is too pessimistic to justify or to explain the actual
computational benefits of random features method in practice.
And the model is formulated as a constrained
optimization problem, which is rarely used in practice.

\cite{cortes2010} and \cite{Sutherland2015}
considered the performance of RFSVM as a perturbed optimization problem, using
the fact that the dual form of KSVM is a constrained quadratic optimization
problem. Although the maximizer of a quadratic function depends continuously
on the quadratic form, its dependence is weak and thus, both papers failed to
obtain an informative bound for the excess risk of RFSVM in the
classification problem. In particular, such an approach requires RFSVM and
KSVM to be compared under the same hyper-parameters.
This assumption is, in fact, problematic because the optimal configuration of
hyper-parameters of RFSVM is not necessarily the same as those for the corresponding KSVM.
In this sense, RFSVM is more like an independent learning model instead of
just an approximation to KSVM.

In regression settings, the learning rate of random features method was studied by
\cite{Rudi2017} under the assumption that the regression function is in the
RKHS, namely the \emph{realizable} case.
They show that the uniform feature sampling only
requires $ O(\sqrt{m}\log(m)) $ features to achieve $ O(1/\sqrt{m}) $ risk of squared
loss. They further show that a data-dependent sampling can achieve
a rate of $ O(1/m^{\alpha}) $, where $ 1/2\le\alpha\le 1 $, with even fewer
features, when the regression function is sufficiently smooth and the spectrum
of the kernel integral operator decays sufficiently fast. However, the method
leading to these results depends on the closed form of the least squares solution,
and thus we cannot easily extend these results to non-smooth loss functions used in RFSVM.
\cite{Bach2017} recently shows that for any given approximation accuracy,
the number of random features required is given by the degrees of freedom of the
kernel operator under such an accuracy level, when optimized features are available.
This result is crucial for sample
complexity analysis of RFSVM, though not many details are provided on this topic
in Bach's work.

In this paper, we investigate the performance of RFSVM formulated as a regularized
optimization problem on classification tasks. In contrast to the slow learning rate in previous
results by \cite{Rahimi2009} and \cite{Bach2017}, we show, for the first
time, that RFSVM can achieve fast learning rate with far fewer features
than the number of samples when the optimized features (see Assumption~\ref{opt-features})
are available, and thus we justify the potential computational benefits
of RFSVM on classification tasks. We mainly considered two learning scenarios:
the realizable case, and then unrealizable case, where the Bayes classifier
does not belong to the RKHS of the feature
map. In particular, our contributions are threefold:
\begin{enumerate}
    \item We prove that under Massart's low noise condition, with an optimized feature map,
    RFSVM can achieve a learning rate of $ \tilde{O}(m^{-\frac{c_2}{1+c_2}}) $
    \footnote{$ \tilde{O}(n) $ represents a quantity less than $ Cn\log^k(n) $
    for some $ k $. },
    with $ \tilde{O}(m^{\frac{2}{2+c_2}}) $
    number of features when the Bayes classifier belongs to the RKHS of a
    kernel whose spectrum decays polynomially ($ \lambda_i = O(i^{-c_2} $)).
    When the decay rate of the spectrum of kernel operator is sub-exponential, the
    learning rate can be improved to $ \tilde{O}(1/m) $ with only
    $ \tilde{O}(\ln^{d}(m)) $ number of features.
    \item When the Bayes classifier satisfies the separation condition; that is,
    when the two classes of points are apart by a positive distance,
    we prove that the RFSVM using an optimized feature
    map corresponding to Gaussian kernel can achieve a learning rate
    of $ \tilde{O}(1/m) $ with $ \tilde{O}(\ln^{2d}(m)) $ number of features.
    \item Our theoretical analysis suggests reweighting random features before training.
    We confirm its benefit in our experiments over synthetic data sets.
\end{enumerate}

We begin in Section~\ref{sec:prelim} with a brief introduction of RKHS,
random features and the problem formulation, and set up the
notations we use throughout the rest of the paper. In
Section~\ref{sec:main_results}, we provide our main theoretical results
(see the appendices for the proofs), and
in Section~\ref{sec:experiments}, we verify the performance of RFSVM in experiments.
In particular, we show the improvement brought by the reweighted feature selection
algorithm. The conclusion and some open questions are summarized at the end.
The proofs of our main theorems follow from a combination of the sample complexity analysis
scheme used by \cite{Steinwart2008} and the approximation error result of
\cite{Bach2017}. The fast rate is achieved due to the fact that the Rademacher
complexity of the RKHS of $ N $ random features and with regularization parameter
$ \lambda $ is only $ O(\sqrt{N\log(1/\lambda)}) $, while $ N $ and $ 1/\lambda $
need not be too large to control the approximation error
when optimized features are available. Detailed proofs and more experimental results are provided
in the Appendices for interested readers.

\section{Preliminaries and notations}
\label{sec:prelim}

Throughout this paper, a labeled data point is a point $(x,y)$ in
$\mathcal{X}\times\left\{ -1,1\right\} $, where $\mathcal{X}$ is a bounded
subset of $\r^{d}$. $ \mathcal{X}\times\{-1,1\} $ is equipped with
a probability distribution $ \mathbb{P} $.

\subsection{Kernels and Random Features}
A positive definite
kernel function $k\left(x,x'\right)$ defined on $\mathcal{X}\times\mathcal{X}$
determines the unique corresponding reproducing kernel
Hilbert space (RKHS), denoted by $\mathcal{F}_k$. A map $\phi$ from the
data space $\mathcal{X}$ to a Hilbert space $H$ such that
$\ip{\phi\left(x\right)}{\phi\left(x'\right)}_{H}=k\left(x,x'\right)$
is called a feature map of $ k $ and $H$ is called a feature space.
For any $f\in\mathcal{F}$, there exists
an $h\in H$ such that $\langle h,\phi(x)\rangle_H=f(x)$, and
the infimum of the norms of all such $ h $s is equal to $\Vert f\Vert_\mathcal{F}$.
On the other hand, given any feature map $ \phi $ into $ H $, a kernel function
is defined by the equation above, and we call $ \mathcal{F}_k $ the RKHS
corresponding to $ \phi $, denoted by $ \mathcal{F}_\phi $.

A common choice of feature space is the $ L^2 $ space of a probability space
$ (\omega,\Omega,\nu) $. An important observation is that for any probability
density function $ q(\omega) $ defined on $ \Omega $,
$ \phi(\omega;x)/\sqrt{q(\omega)} $ with probability
measure $ q(\omega)\mathrm{d}\nu(\omega) $ defines the same kernel function with
the feature map $ \phi(\omega;x) $ under the distribution $ \nu $.
One can sample the image of $ x $ under the feature map $ \phi $,
an $ L^2 $ function $ \phi(\omega;x) $, at points $ \{\omega_1,\ldots,\omega_N\} $
according to the probability distribution $ \nu $ to approximately represent $ x $.
Then the vector in $ \r^N $ is called a random feature vector of $ x $,
denoted by $ \phi_N(x) $. The corresponding kernel function determined by $ \phi_N $
is denoted by $ k_N $.

A well-known construction of random features is the random Fourier features proposed
by \cite{Rahimi2008}. The feature map is defined as follows,
\begin{align*}
  \phi:\mathcal{X} & \to L^2(\r^d,\nu)\oplus L^2(\r^d,\nu) \\
  x & \mapsto \left(\cos\left(\omega\cdot x\right),
    \sin\left(\omega\cdot x\right)\right)\,.
\end{align*}
And the corresponding random feature vector is
\begin{equation*}
    \phi_N(x) = \frac{1}{\sqrt{N}}\left(
    \cos\left(\omega\cdot x\right),\cdots,
    \cos\left(\omega\cdot x\right),
    \sin\left(\omega\cdot x\right),\cdots,
    \sin\left(\omega\cdot x\right)
    \right)^{\intercal}\,,
\end{equation*}
where $ \omega_i $s are sampled according to $ \nu $. Different choices of
$ \nu $ define different translation invariant kernels (see \cite{Rahimi2008}).
When $ \nu $ is the normal distribution with mean $ 0 $ and variance $ \gamma^{-2} $,
the kernel function defined by the feature map is Gaussian kernel with bandwidth
parameter $ \gamma $,
\begin{equation*}
  k_\gamma(x,x') = \exp\left(-\frac{\Vert x-x'\Vert^2}{2\gamma^2}\right)\,.
\end{equation*}
Equivalently, we may consider the feature map $ \phi_\gamma(\omega;x):=\phi(\omega/\gamma;x) $
with $ \nu $ being standard normal distribution.

A more general and more abstract feature map can be constructed using an orthonormal
set of $ L^2(\mathcal{X},\p_\mathcal{X}) $. Given the orthonormal set
$ \{e_i\} $ consisting of bounded functions,
and a nonnegative sequence $ (\lambda_i) \in \ell^1 $,
we can define a feature map
\begin{equation*}
  \phi(\omega;x)=\sum_{i=1}^\infty \sqrt{\lambda_i}e_i(x)e_i(\omega)\,,
\end{equation*}
with feature space $ L^2(\omega,\mathcal{X},\mathbb{P}_\mathcal{X}) $.
The corresponding kernel is given by
$k(x,x')=\sum_{i=1}^\infty \lambda_i e_i(x)e_i(x')$.
The feature map and the kernel function are well defined because of the boundedness
assumption on $ \{ e_i \} $.  A similar representation can be obtained for a
continuous kernel function on a compact set by Mercer's Theorem (\cite{Lax2002}).

Every positive definite kernel function $ k $ satisfying that
$ \int k(x,x)~\d\p_\mathcal{X}(x) < \infty $
defines an integral operator on $ L^2(x,\mathcal{X},\p_\mathcal{X}) $ by
\begin{align*}
  \Sigma:L^2(\mathcal{X},\mathbb{P}_\mathcal{X}) & \to
  L^2(\mathcal{X},\mathbb{P}_\mathcal{X}) \\
        f &\mapsto \int_\mathcal{X} k(x,t)f(t)~
        \mathrm{d}\mathbb{P}_\mathcal{X}(t)\,.
\end{align*}
$ \Sigma $ is of trace class with trace norm $ \int k(x,x)~\d\p_\mathcal{X}(x) $.
When the integral operator is determined by a feature map $ \phi $, we denote
it by $ \Sigma_\phi $, and the $ i $th eigenvalue in a descending order by
$ \lambda_i(\Sigma_\phi) $. Note that the regularization paramter is also denoted
by $ \lambda $ but without a subscript. The decay rate of the spectrum of $ \Sigma_\phi $ plays an
important role in the analysis of learning rate of random features method.

\subsection{Formulation of Support Vector Machine}
Given $m$ samples $\{ (x_{i},y_{i})\} _{i=1}^{m}$ generated
i.i.d. by $\p$ and a function $f:\mathcal{X}\to\r$, usually called a hypothesis in the
machine learning context, the empirical and expected risks with respect to
the loss function $\ell$ are defined by
\[
R_{m}^{\ell}\left(f\right):=
\frac{1}{m}\sum_{i=1}^{m}\ell\left(y_{i},f\left(x_{i}\right)\right)
\quad
R_{\p}^{\ell}\left(f\right):=
\e_{\left(x,y\right)\sim\p}\ell\left(y,f\left(x\right)\right)\,,
\]
respectively.

The 0-1 loss is commonly used to measure the performance of classifiers:
\[
\ell^{\mathrm{0-1}}(y,f(x)) =
\begin{cases}
    1 & \text{if } f(x)y \le 0;\\
    0 & \text{if } f(x)y > 0.
\end{cases}
\]
The function that minimizes the expected risk under 0-1 loss
is called the Bayes classifier, defined by
\[
    f^*_{\p}(x):=\operatorname{sgn}\left(\e[y\mid x]\right)\,.
\]
The goal of the classification task is to find a good hypothesis $f$ with small
excess risk $R_{\p}^{\mathrm{0-1}}(f)-R_{\p}^{\mathrm{0-1}}(f^*_{\p})$.
And to find the good hypothesis based on the samples, one minimizes the empirical risk.
However, using 0-1 loss, it is hard to find the global minimizer of
the empirical risk because the loss function is discontinuous and non-convex.
A popular surrogate loss function in practice is the hinge loss:
$ \ell^h(f) = \max(0,1-yf(x)) $, which guarantees that
\begin{equation*}
  R^{h}_\p(f) - \inf_f R^{h}_{\p}(f) \ge
  R^{\mathrm{0-1}}_{\p}(f) - R^{\mathrm{0-1}}_{\p}(f^*_{\p})\,,
\end{equation*}
where $ R^h $ means $ R^{\ell^h} $ and $ R^{\mathrm{0-1}} $ means $ R^{\ell^{\mathrm{0-1}}} $.
See \cite{Steinwart2008} for more details.

A regularizer can be added into the optimization objective
with a scalar multiplier $ \lambda $ to avoid overfitting the random samples.
Throughout this paper, we consider the most commonly used $ \ell^2 $ regularization.
Therefore, the solution of the binary classification problem is given by minimizing
the following objective
\begin{equation*}
    R_{m,\lambda}(f)=R^h_m(f)+\frac{\lambda}{2}\Norm f_{\mathcal{F}}^2\,,
\end{equation*}
over a hypothesis class $ \mathcal{F} $. When $ \mathcal{F} $ is the RKHS of
some kernel function, the algorithm described above is called kernel support
vector machine. Note that for technical convenience, we do not include the bias
term in the formulation of hypothesis so that all these functions are from the
RKHS instead of the product space of RKHS and $ \mathbb{R} $ (see Chapter 1
of \cite{Steinwart2008} for more explanation of such a convention).
Note that $R_{m,\lambda}$ is strongly convex and thus
the infimum will be attained by some function in $\mathcal{F}$. We denote it by
$f_{m,\lambda}$.

When random features $ \phi_N $ and the corresponding RKHS are considered,
we add $ N $ into the subscripts of the notations defined above to
indicate the number of random features. For example $ \mathcal{F}_N $ for the RKHS,
$ f_{N,m,\lambda} $ for the solution of the optimization problem.

\section{Main Results}
\label{sec:main_results}

In this section we state our main results on the fast learning rates of RFSVM
in different scenarios.

First, we need the following assumption on the distribution of data, which is
required for all the results in this paper.
\begin{assumption}
  \label{massart}
  There exists $ V\ge 2 $ such that
  \begin{equation*}
    \vert\e_{(x,y)\sim\p}[y\mid x]\vert \ge 2/V\,.
  \end{equation*}
\end{assumption}
This assumption is called Massart's low noise condition in many references (see
for example \cite{Koltchinskii2011}). When $ V=2 $ then all the data points
have deterministic labels almost surely. Therefore it is easier to learn the
true classifier based on observations. In the proof, Massart's low noise
condition guarantees the variance condition (\cite{Steinwart2008})
\begin{equation}
  \label{eq:variance}
  \e[(\ell^h(f(x))-\ell^h(f^*_{\p}(x)))^2] \le V(R^h(f) - R^h(f^*_{\p}))\,,
\end{equation}
which is a common requirement for the fast rate results. Massart's condition is
an extreme case of a more general low noise condition, called Tsybakov's condition.
For the simplicity of the theorem, we only consider Massart's condition in our
work, but our main results can be generalized to Tsybakov's condition.

The second assumption is about the quality of random features. It was first introduced
in \cite{Bach2017}'s approximation results.
\begin{assumption}
  \label{opt-features}
  A feature map $ \phi:\mathcal{X}\to L^2(\omega,\Omega,\nu)) $ is called optimized if
  there exists a small constant $ \mu_0 $ such that for any $ \mu\le\mu_0 $,
  \begin{equation*}
    \sup_{\omega\in\Omega} \Vert (\Sigma+\mu I)^{-1/2}\phi(\omega;x) \Vert_{L^2(\mathbb{P})}^2
    \le \mathrm{tr}(\Sigma(\Sigma + \mu I)^{-1})
    = \sum_{i=1}^\infty \frac{\lambda_i(\Sigma)}{\lambda_i(\Sigma) + \mu}\,.
  \end{equation*}
\end{assumption}
For any given $ \mu $, the quantity on the left hand side of the inequality is called
leverage score with respect to $ \mu $, which is directly related with
the number of features required to approximate
a function in the RKHS of $ \phi $. The quantity on the right hand
side is called degrees of freedom by \cite{Bach2017} and effective dimension by
\cite{Rudi2017}, denoted by $ d(\mu) $. Note that whatever the RKHS is, we can
always construct optimized feature map for it. In the Appendix~\ref{app:opt-feature} we describe two
examples of constructing optimized feature map. When a feature map is optimized, it is easy to control its
leverage score by the decay rate of the spectrum of $ \Sigma $, as described below.

\begin{definition}
  We say that the spectrum of $ \Sigma:L^2(\mathcal{X},\p)\to L^2(\mathcal{X},\p) $
  decays at a polynomial rate if there exist $ c_1>0 $ and $ c_2>1 $ such that
  \begin{equation*}
    \lambda_i(\Sigma) \le c_1 i^{-c_2}\,.
  \end{equation*}
  We say that it decays sub-exponentially if there exist $ c_3,c_4 > 0 $ such that
  \begin{equation*}
    \lambda_i(\Sigma) \le c_3 \exp(-c_4 i^{1/d})\,.
  \end{equation*}
\end{definition}
The decay rate of the spectrum of $ \Sigma $ characterizes the capacity
of the hypothesis space to search for the solution, which further determines the
number of random features required in the learning process. Indeed, when the
feature map is optimized, the number of features required to approximate a function
in the RKHS with accuracy $ O(\sqrt{\mu}) $ is upper bounded by $ O(d(\mu)\ln(d(\mu))) $.
When the spectrum decays polynomially, the degrees of freedom $ d(\mu) $ is $ O(\mu^{-1/c_2}) $,
and when it decays sub-exponentially, $ d(\mu) $ is $ O(\ln^d(c_3/\mu)) $
(see \lemref{lem:dof} in Appendix~\ref{app:realizable} for details). 
Examples on the kernels with polynomial and sub-exponential spectrum decays can be found
in \cite{Bach2017}. Our proof of Lemma~\ref{lem:spectrum} also provides some
useful discussion.

With these preparations, we can state our first theorem now.
\begin{theorem}
  \label{thm:realizable}
  Assume that $ \p $ satisfies Assumption~\ref{massart}, and the feature map
  $ \phi $ satisfies Assumption~\ref{opt-features}. If $ f^*_\p\in\mathcal{F}_\phi $
  with $ \Vert f^*_\p\Vert_{\mathcal{F}_\phi} \le R $. Then when
  the spectrum of $ \Sigma_\phi $ decays polynomially, by choosing
  \begin{align*}
    \lambda & = m^{-\frac{c_2}{2+c_2}} \\
    N & = 10C_{c_1,c_2}m^{\frac{2}{2+c_2}}(\ln(32C_{c_1,c_2}m^{\frac{2}{2+c_2}})+\ln(1/\delta))\,,
  \end{align*}
  we have
  \begin{equation*}
    R^{\mathrm{0-1}}_\p(f_{N,m,\lambda}) - R^{\mathrm{0-1}}_\p(f^*_\p)
    \le C_{c_1,c_2,V,R}m^{-\frac{c_2}{2+c_2}}\left((\ln(1/\delta)+\ln(m))\right)\,,
  \end{equation*}
  with probability $ 1-4\delta $.
  When the spectrum of $ \Sigma_\phi $ decays sub-exponentially, by choosing
  \begin{align*}
    \lambda & =1/m \\
    N & = 25C_{d,c_4}\ln^d(m)(\ln(80C_{d,c_4}\ln^d(m))+\ln(1/\delta))\,,
  \end{align*}
  we have
  \begin{align*}
    R^{\mathrm{0-1}}_\p(f_{N,m,\lambda}) - R^{\mathrm{0-1}}_\p(f^*_\p)
    & \le C_{c_3,c_4,d,R,V}\frac{1}{m}\left(\log^{d+2}(m)+\log(1/\delta)\right)\,,
  \end{align*}
  with probability $ 1-4\delta $ when $ m \ge \exp((c_4\vee\frac{1}{c_4})d^2/2) $.
\end{theorem}
This theorem characterizes the learning rate of RFSVM in realizable cases; that is, when
the Bayes classifier belongs to the RKHS of the feature map. For polynomially
decaying spectrum, when $ c_2>2 $, we get a learning rate faster than $ 1/\sqrt{m} $.
\cite{Rudi2017} obtained a similar fast
learning rate for kernel ridge regression with random features (RFKRR), assuming
polynomial decay of the spectrum of $ \Sigma_\phi $ and the existence of a minimizer
of the risk in $ \mathcal{F}_\phi $. Our theorem extends their result to
classification problems and exponential decay spectrum. However, we have to
use a stronger assumption that $ f^*_\p \in \mathcal{F}_\phi $ so that the low noise condition can
be applied to derive the variance condition. For RFKRR, the rate faster than
$ O(1/\sqrt{m}) $ will be achieved whenever $ c_2>1 $, and the number of features
required is only square root of our result. We think that this is mainly caused
by the fact that their surrogate loss is squared. The result for the sub-exponentially
decaying spectrum is not investigated for RFKRR, so we cannot make a comparison.
We believe that this is the first result showing that RFSVM can achieve $ \tilde{O}(1/m) $
with only $ \tilde{O}(\ln^d(m)) $ features. Note however that when $ d $ is large,
the sub-exponential case requires a large number of samples, even possibly larger
than the polynomial case. This is clearly an artifact of our analysis since we can always use
the polynomial case to provide an upper bound! We therefore suspect that there is considerable room for
improving our analysis of high dimensional data in the sub-exponential decay case.
In particular, removing the exponential dependence on $ d $ under reasonable assumptions
is an interesting direction for future work.

To remove the realizability assumption, we provide our second theorem,
on the learning rate of RFSVM in unrealizable case. We focus
on the random features corresponding to
the Gaussian kernel as introduced in Section~\ref{sec:prelim}.
When the Bayes classifier does not belong to the RKHS,
we need an approximation theorem to estimate the gap of risks. The
approximation property of RKHS of Gaussian kernel has been studied in
\cite{Steinwart2008}, where the margin noise exponent is defined to derive
the risk gap. Here we introduce the simpler and stronger separation condition,
which leads to a strong result.

The points in $\mathcal{X}$ can be collected in to two sets according to
their labels as follows,
\begin{align*}
    \mathcal{X}_1 & :=\{x\in\mathcal{X}\mid \e(y\mid x)>0\} \\
    \mathcal{X}_{-1} & :=\{x\in\mathcal{X}\mid \e(y\mid x)<0\}\,.
\end{align*}
The distance of a point $ x\in \mathcal{X}_i $ to the set $ \mathcal{X}_{-i} $ is denoted
by $ \Delta(x) $.
\begin{assumption}
  \label{separation}
  We say that the data distribution satisfies a separation condition if there exists
  $ \tau > 0 $ such that $ \p_\mathcal{X}({\Delta(x)<\tau})=0 $.
\end{assumption}
Intuitively, Assumption~\ref{separation} requires the two classes to be
far apart from each other almost surely. This separation assumption
is an extreme case when the margin noise exponent goes to infinity.

The separation condition characterizes a different aspect of data distribution from
Massart's low noise condition. Massart's low noise condition guarantees that
the random samples represent the distribution behind them accurately,
while the separation condition guarantees
the existence of a smooth, in the sense of small derivatives, function achieving
the same risk with the Bayes classifier.

With both assumptions imposed on $ \p $, we can get a fast learning rate
of $ \ln^{2d+1}m/m $ with only $ \ln^{2d}(m) $ random features, as stated in the
following theorem.
\begin{theorem}
    \label{thm:unrealizable}
    Assume that $ \mathcal{X} $ is bounded by radius $ \rho $. The data
    distribution has density function upper bounded by a constant $ B $, and satisfies
    Assumption~\ref{massart} and \ref{separation}. Then by choosing
    \begin{equation*}
      \lambda=1/m \quad \gamma=\tau/\sqrt{\ln m}
      \quad N=C_{\tau,d,\rho}\ln^{2d}m(\ln\ln m+\ln(1/\delta)) \,,
    \end{equation*}
    the RFSVM using an optimized feature map corresponding to the Gaussian kernel
    with bandwidth $ \gamma $ achieves the learning rate
    \begin{equation*}
      R^{\mathrm{0-1}}_\p(f_{N,m,\lambda}) - R^{\mathrm{0-1}}_\p(f^*_\p)
      \le C_{\tau,V,d,\rho,B}\frac{\ln^{2d+1}(m)(\ln\ln(m) + \ln(1/\delta))}{m}\,,
    \end{equation*}
    with probability greater than $ 1-4\delta $ for $ m\ge m_0 $, where $ m_0 $ depends
    on $ \tau,\rho,d $.
\end{theorem}
To the best of our knowledge, this is the first theorem on the fast learning rate
of random features method in the unrealizable case. It only
assumes that the data distribution satisfies low noise and separation conditions,
and shows that with an optimized feature distribution, the learning rate of $ \tilde{O}(1/m) $
can be achieved using only $ \ln^{2d+1}(m)\ll m $ features. This justifies the benefit
of using RFSVM in binary classification problems. The assumption of a bounded
data set and a bounded distribution density function can be dropped if we assume
that the probability density function is
upper bounded by $ C\exp(-\gamma^2\Vert x\Vert^2/2) $, which suffices to provide
the sub-exponential decay of spectrum of $ \Sigma_\phi $. But we prefer the simpler
form of the results under current conditions. We speculate that the conclusion of
\thmref{thm:unrealizable} can be generalized to all sub-Gaussian data.

The main drawback of our two theorems is the assumption of an optimized
feature distribution, which is hard to obtain in practice.
Developing a data-dependent feature selection method is therefore an important
problem for future work on RFSVM. \cite{Bach2017} proposed an algorithm to
approximate the optimized feature map from any feature map.
Adapted to our setup, the reweighted feature selection algorithm is described as follows.
\begin{enumerate}
    \item Select $M$ i.i.d. random vectors $\{\omega_i\}_{i=1}^M$ according to
    the distribution $d\nu_{\gamma}$.
    \item Select  $L$ data points $\{x_i\}_{i=1}^{L}$ uniformly from the training set.
    \item Generate the matrix $\Phi$ with columns $\phi_M(x_i)/\sqrt{L}$.
    \item Compute $ \{r_i\}_{i=1}^M $, the diagonal of $\Phi\Phi^{\t}(\Phi\Phi^{\t}+\mu I)^{-1}$.
    \item Resample $N$ features from $\{\omega_i\}_{i=1}^M$ according to the
    probability distribution $p_i=r_i/\sum r_i $.
\end{enumerate}
The theoretical guarantees of this algorithm have not been discussed in the literature.
A result in this direction will be extremely useful for guiding practioners.
However, it is outside the scope of our work. Instead, here we
implement it in our experiment and empirically compare the performance of RFSVM using
this reweighted feature selection method to the performance of RFSVM
without this preprocessing step; see Section~\ref{sec:experiments}.

For the realizable case, if we drop the assumption of optimized feature map, only weak results can be
obtained for the learning rate and the number of features required
(see Appendix~\ref{app:unif} for more details). In particular, we can only show that $ 1/\epsilon^2 $
random features are sufficient to guarantee the learning rate less than $ \epsilon $ when
$ 1/\epsilon^3 $ samples are available. Though not helpful for justifying
the computational benefit of random features method, this result matches the
parallel result for RFKRR in \cite{Rudi2017} and the approximation result in
\cite{Sriperumbudur2015}. We conjecture that this upper bound is also optimal for RFSVM.

\cite{Rudi2017} also compared the performance of RFKRR with Nystrom method, which
is the other popular method to scale kernel ridge regression to large data sets.We do not
find any theoretical guarantees on the fast learning rate of SVM with Nystrom method on classification
problems in the literature, though there are several works on its approximation quality
to the accurate model and its empirical performance
(see \cite{Yang2012,Zhang2012}). The tools used in this paper should also work for
learning rate analysis of SVM using Nystrom method. We leave this analysis to the future.

\section{Experimental Results}
\label{sec:experiments}

In this section we evaluate the performance of RFSVM with the reweighted feature selection
algorithm\footnote{The source code is available at
\url{https://github.com/syitong/randfourier}.}. The sample points shown in Figure~\ref{fig:samples}
are generated from either the inner circle or outer annulus uniformly
with equal probability, where the radius of the inner circle is 0.9, and the radius
of the outer annulus ranges from 1.1 to 2. The points from the inner circle are labeled by -1 with
probability 0.9, while the points from the outer annulus are labeled by 1 with
probability 0.9. In such a simple case, the unit circle describes the Bayes
classifier.

First, we compared the performance of RFSVM with that of KSVM on the training
set with $1000$ samples, over
a large range of regularization parameter ($-7\le\log\lambda\le1$). The bandwidth parameter
$\gamma$ is fixed to be an estimate of the average distance among the training
samples. After training, models are tested on a large testing set ($>10^5$).
For RFSVM, we considered the effect of the number of features by setting $N$ to be
$1,3,5,10$ and $20$, respectively. Moreover, both feature selection methods,
simple random feature selection (labeled by `unif' in the figures), which does not
apply any preprocess on drawing features, and reweighted feature
selection (labeled by `opt' in the figures) are inspected.
For the reweighted method, we set $M=100N$ and $L=0.3m$ to compute the weight of each feature.
Every RFSVM is run 10 times, and the average accuracy and standard deviation are
presented.

The results of KSVM, RFSVMs with 1 and 20 features are shown in
\figureref{fig:N1} and \figureref{fig:N20} respectively
(see the results of other levels of features
in Appendix~\ref{app:figures} in the supplementary material).
The performance of RFSVM is slightly worse than
the KSVM, but improves as the number of features increases. It
also performs better when the reweighted method is applied to generate
features.
\begin{figure}
    \begin{minipage}{0.48\textwidth}
      {\centering
      \includegraphics[width=\linewidth]{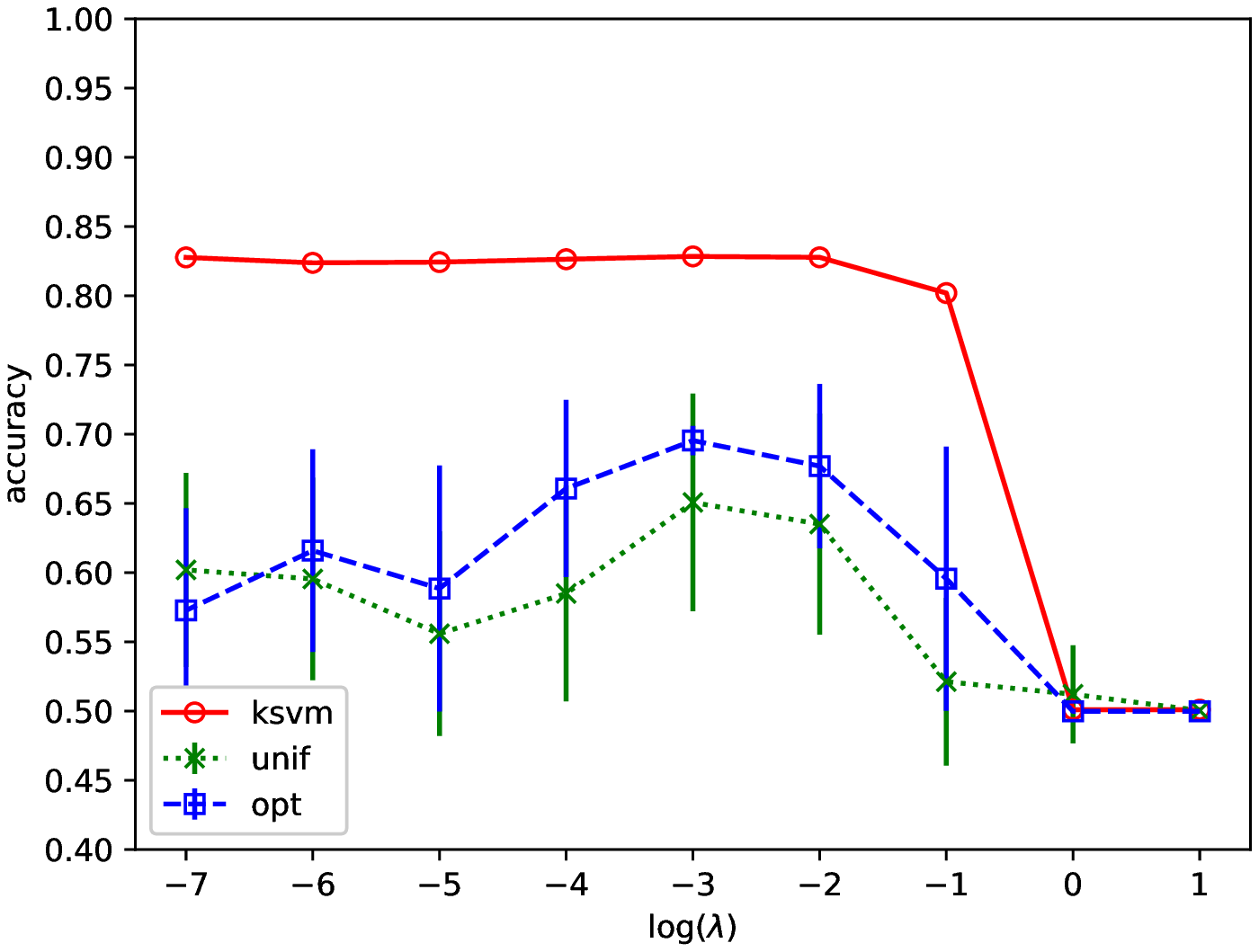}
      \caption{RFSVM with 1 feature. }
      \label{fig:N1}}
    \end{minipage}%
    \hfill{}
    \begin{minipage}{0.48\textwidth}
      {\centering
      \includegraphics[width=\linewidth]{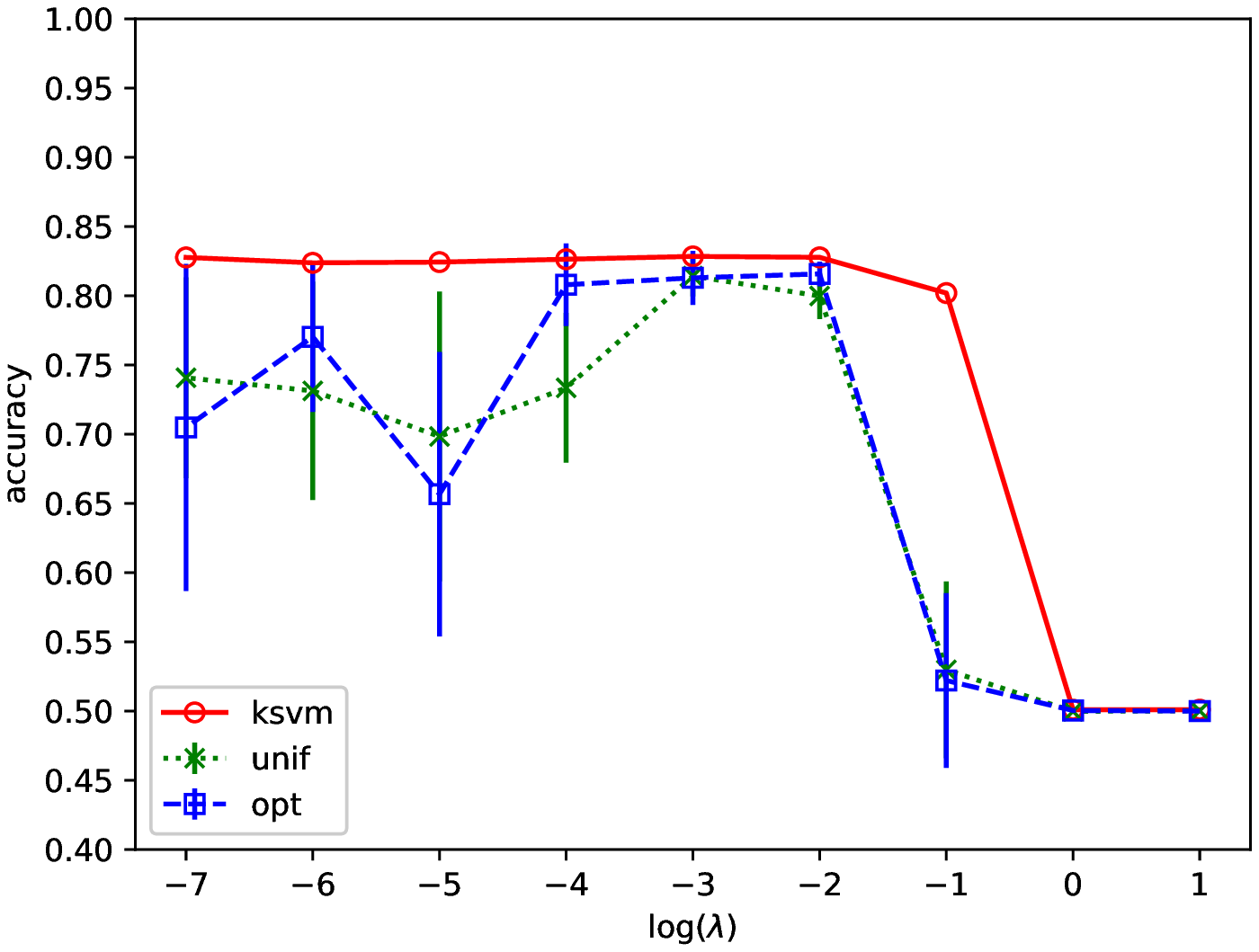}
      \caption{RFSVM with 20 features. }
      \label{fig:N20}}
    \end{minipage}%

    \medskip
    \small
    ``ksvm'' is for KSVM with Gaussian kernel,
    ``unif'' is for RFSVM with direct feature sampling,
    and ``opt'' is for RFSVM with reweighted feature sampling.
    Error bars represent standard deviation over 10 runs.
\end{figure}
\begin{figure}
    \centering
    \begin{minipage}{0.48\textwidth}
      {\centering
      \includegraphics[width=\linewidth]{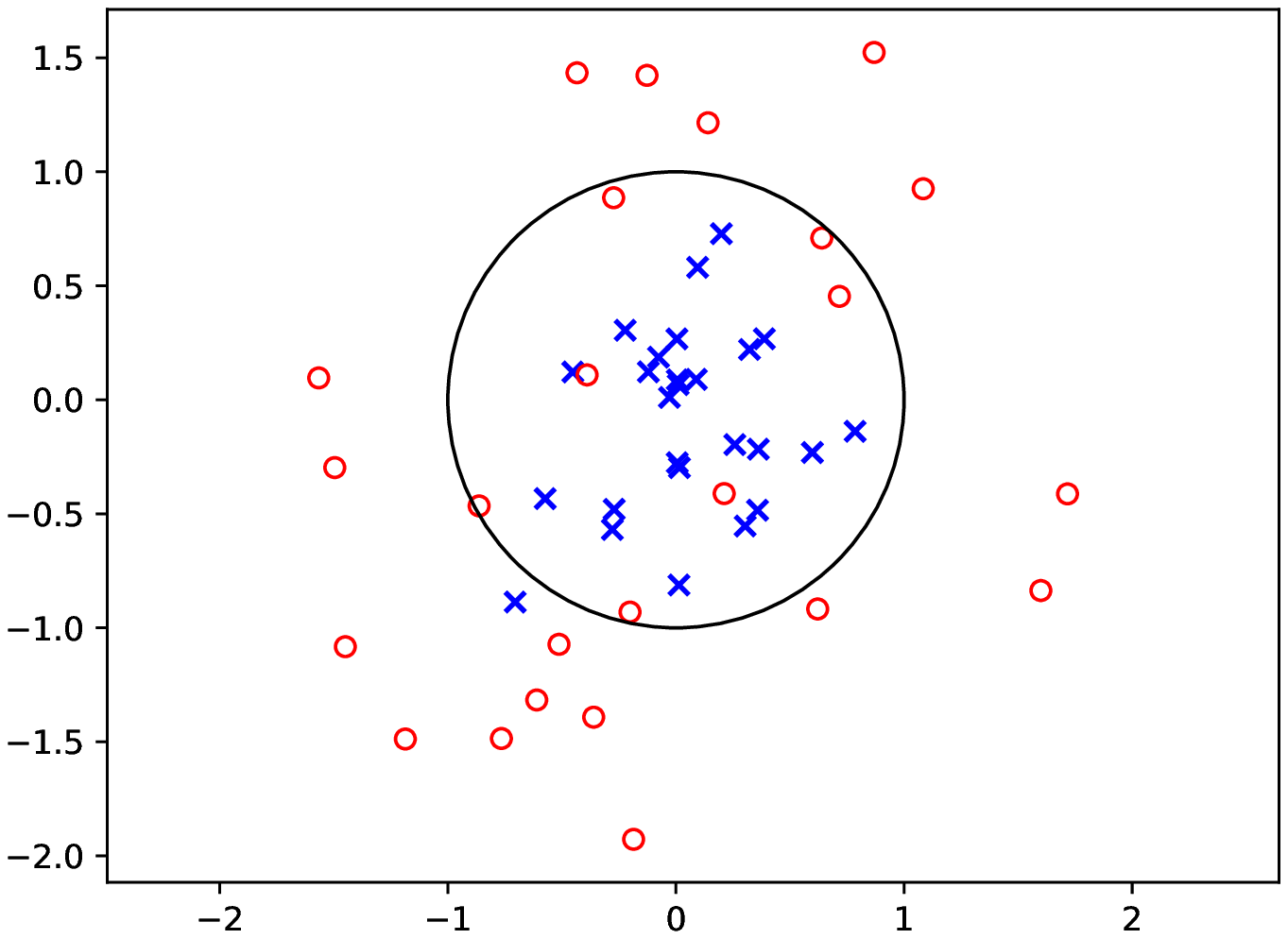}
      \caption{Distribution of Training Samples. }
      \label{fig:samples}}
      \medskip
      \small
      50 points are shown
      in the graph. Blue crosses represent the points labeled by -1,
      and red circles the points labeled by $1$. The unit circle is
      one of the best classifier for these data with 90\% accuracy.
    \end{minipage}%
    \hfill{}
    \begin{minipage}{0.48\textwidth}
      {\centering
      \includegraphics[width=\linewidth]{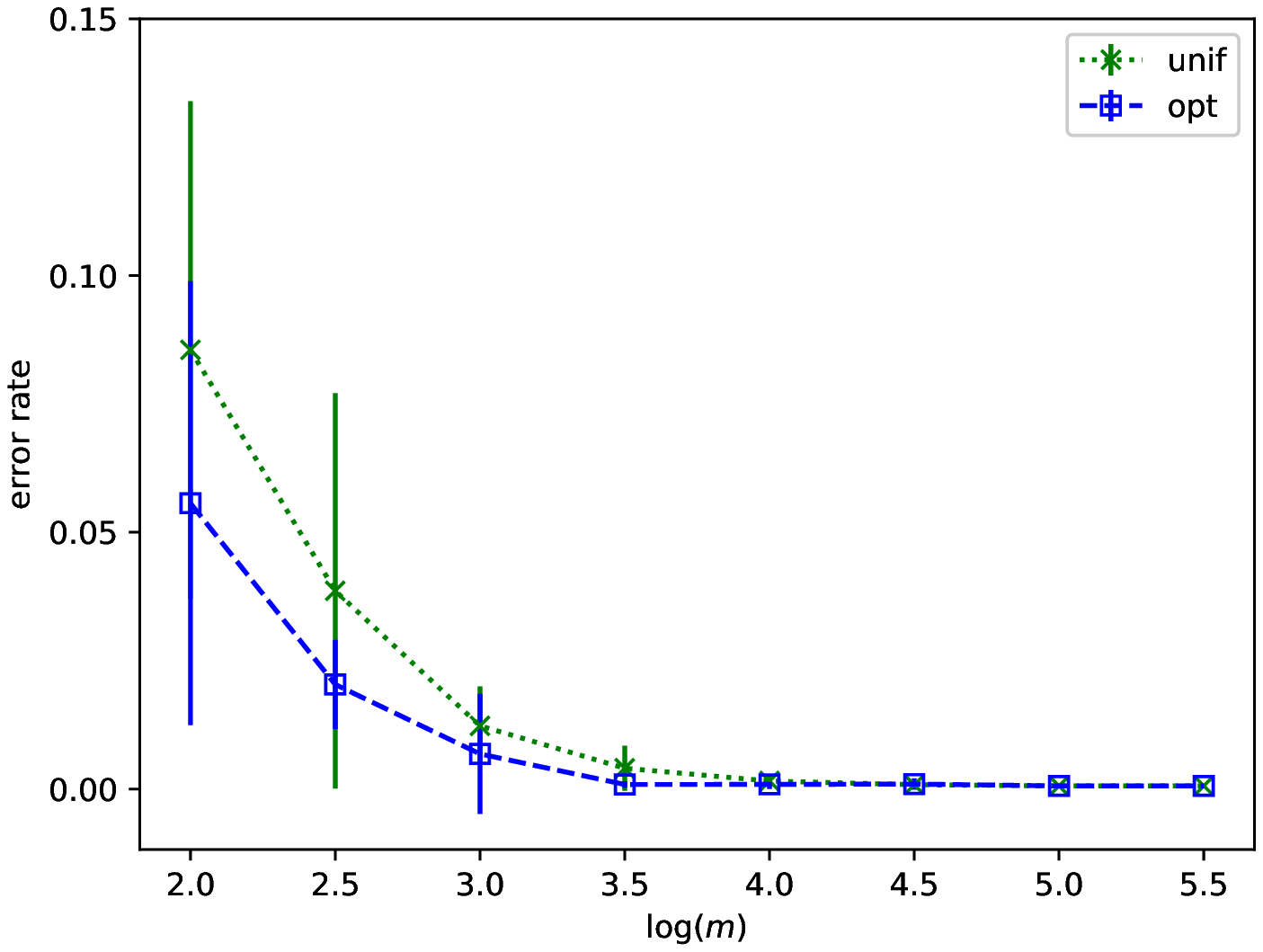}
      \caption{Learning Rate of RFSVMs. }
      \label{fig:LearningRate}}
      \medskip
      \small
      The excess risks of RFSVMs
      with the simple random feature selection (``unif'')
      and the reweighted feature selection (``opt'') are shown for different sample
      sizes. The error rate is the excess risk. The error bars represent the
      standard deviation over 10 runs.
    \end{minipage}
\end{figure}

To further compare the performance of simple feature selection and reweighted
feature selection methods, we plot the learning rate of RFSVM with $O(\ln^{2}(m))$ features
and the best $\lambda$s for each sample size $ m $. KSVM is not included here since it is too slow
on training sets of size larger than $10^4$ in our experiment compared to RFSVM.
The error rate in \figureref{fig:LearningRate} is
the excess risk between learned classifiers and the Bayes classifier. We can
see that the excess risk decays as $m$ increases, and the RFSVM using reweighted
feature selection method outperforms the simple feature selection.

According to Theorem~\ref{thm:unrealizable}, the benefit brought by optimized feature
map, that is, the fast learning rate, will show up when the sample size is
greater than $ O(\exp(d)) $ (see Appendix~\ref{app:unrealizable}). The number
of random features required also depends on $ d $, the dimension of data. For data of
small dimension and large sample size, as in our experiment, it is not a problem.
However, in applications of image recognition, the dimension of the data is
usually very large and it is hard for our theorem to explain the performance
of RFSVM. On the other hand, if we do not pursue the fast learning rate,
the analysis for general feature maps, not necessarily optimized, gives a
learning rate of $ O(m^{-1/3}) $ with $ O(m^{2/3}) $ random features, which
does not depend on the dimension of data (see Appendix~\ref{app:unif}).
Actually, for high dimensional data,
there is barely any improvement in the performance of RFSVM by using reweighted
feature selection method (see Appendix~\ref{app:figures}).
It is important to understand the role of $ d $ to fully
understand the power of random features method.

\section{Conclusion}

Our study proves that the fast learning rate is possible for RFSVM
in both realizable and unrealizable scenarios when the optimized feature map is available.
In particular, the number of features required is far less than the sample size, which
implies considerably faster training and testing using the random features method.
Moreover, we show in the experiments that
even though we can only approximate the optimized feature distribution
using the reweighted feature selection method, it, indeed, has better performance
than the simple random feature selection. Considering that such a reweighted
method does not rely on the label distribution at all, it will be useful
in learning scenarios where multiple classification problems share the same
features but differ in the class labels.
We believe that a theoretical guarantee of the performance of the reweighted feature
selection method and properly understanding the dependence on the
dimensionality of data are interesting directions for future work.

\subsubsection*{Acknowledgements}
\noindent{AT acknowledges the support of a Sloan Research Fellowship.}

\noindent{ACG acknowledges the support of a Simons Foundation Fellowship.}
\bibliographystyle{plainnat}
\bibliography{mathref}
\newpage{}
\setcounter{page}{1}
\appendix

\section{Examples of Optimized Feature Maps}
\label{app:opt-feature}

Assume that a feature map $ \phi:\mathcal(X)\to L^2(\omega,\Omega,\nu) $ satisfies
that $ \phi(\omega;x) $ is bounded for all $ \omega $ and $ x $. We can always
convert it to an optimized feature map using the method proposed by \cite{Bach2017}.
We rephrase it using our notation as follows.

Define
\begin{equation}
  p(\omega) = \frac{\Vert (\Sigma + \mu I)^{-1/2}\phi(\cdot;\omega)\Vert^2_{L^2(\mathcal{X},\p)}}
  {\int_{\Omega}\Vert (\Sigma + \mu I)^{-1/2}\phi(\cdot;\omega)\Vert^2_{L^2(\mathcal{X},\p)}~\d\nu(\omega)}\,.
\end{equation}
Since $ \phi $ is bounded, its $ L^2 $ norm is finite. The function $ p $
defined above is a probability density function with respect to $ \nu $. 
Then the new feature map is given by
$ \tilde{\phi}(\omega;x) = \phi(\omega;x) / \sqrt{p(\omega)} $ together with the
measure $ p(\omega)\mathrm{d}\nu(\omega) $. With $ \tilde{\phi} $, we have
\begin{align}
  \sup_{\omega\in\Omega} \left\Vert (\Sigma + \mu I)^{-1/2}\tilde{\phi}(\cdot;\omega)\right\Vert^2
  & = \sup_{\omega\in\Omega} \frac{\left\Vert(\Sigma + \mu I)^{-1/2}\phi(\cdot;\omega) \right\Vert^2}
  {p(\omega)} \\
  & = \int_{\Omega}\Vert (\Sigma + \mu I)^{-1/2}
  \phi(\cdot;\omega)\Vert^2_{L^2(\mathcal{X},\p)}~\d\nu(\omega) \\
  & = \mathrm{tr}(\Sigma(\Sigma+\mu I)^{-1})\,.
\end{align}

When the feature map is constructed mapping into $ L^2(\mathcal{X},\p) $ as
described in Section~\ref{sec:prelim}, it is optimized. Indeed, we can compute
\begin{align}
  \sup_{\omega\in\mathcal{X}} \left\Vert (\Sigma + \mu I)^{-1/2}\phi(\cdot;\omega)\right\Vert^2
  & = \sup_{\omega\in\mathcal{X}} \left\Vert \sum_{i=1}^\infty \frac{\sqrt{\lambda_i}}
  {\sqrt{\lambda_i + \mu}}e_i(\cdot)\right\Vert^2 \\
  & = \sum_{i=1}^\infty \frac{\lambda_i}{\lambda_i + \mu}\,.
\end{align}
As an example for this type of feature map, we can consider $ \{ e_i \} $ to be
the Walsh system, which is an orthonormal basis for $ L^2([0,1]) $. Any Bayes
classifier with finitely many discontinuities and discontinuous only at
dyadic, namely points expressable by finite bits, points, will be a finite
linear combination of Walsh basis. This guarantees that the assumptions in
\thmref{thm:realizable} can be satisfied. Our first experiment also make use
of this construction.

The construction above is inspired by the use of spline kernel in \cite{Rudi2017}.
However, our situation is more complicated since the target function, Bayes
classifier, is discontinuous. While the functions in the RKHS generated by the
spline kernel must be continuous (\cite{Cucker2002}). Though we can construct
Bayes classifier using the Walsh basis, we have yet to understand the variety of
possible Bayes classifiers in such a space.

\section{Local Rademacher Complexity of RFSVM}
\label{sec:proof1}

Before the proofs, we first briefly summarize the use of each lemmas and theorems.
Theorem~\ref{thm:est-error} and \ref{thm:app-err} are two fundamental external results for our proof.
Lemma~\ref{lem:app-separation} and \ref{lem:spectrum} refine results that 
appeared in previous works, so that we can apply them to our case.
Lemma~\ref{lem:emprad-RFSVM}, \ref{lem:rad-RFSVM} and \ref{lem:upper-bd-r} are the key results to establish
fast rate for RFSVM, parallel to Steinwarts' work for KSVM. All other smaller and simpler lemmas
included in the appendices are for the purposes of clarity and completeness.
The proofs are not hard but quite technical.

First, both of our theorems are consequences of the following fundamental theorem.
In the theorem, $ \ell^1 $ is the hinge loss clipped at 1, and $ R^1_{\mathbb{P},\lambda} $ 
is the expected regularized risk of $ \ell^1 $.
\begin{theorem}
    \label{thm:est-error}(Theorem 7.20 in \cite{Steinwart2008})
    For a RKHS $ \mathcal{F} $, denote
    $ \inf_{f\in\mathcal{F}} R^1_{\p,\lambda}(f) - R^* $ by $ r^* $.
    For $ r > r^* $, consider the following function classes
    \[
      \mathcal{F}_{r}:=\{f\in\mathcal{F}\mid R_{\p,\lambda}^{1}(f)-R^{*}\le r\}
    \]
    and
    \[
      \mathcal{H}_{r}:=\{\ell^{1}\circ f-\ell^{1}\circ f_{\p}^{*}
      \mid f\in\mathcal{F}_{r}\}\,.
    \]
    Assume that there exists $ V \ge 1 $ such that for any $ f \in \mathcal{F} $,
    \[
      \e_{\p}(\ell^{1}\circ f - \ell^{1}\circ f_{\p}^{*})^{2}
      \le V(R_{\p}^{1}(f)-R^{*})\,.
    \]
    If there is a function $\varphi_{m}:[0,\infty)\to[0,\infty)$
    such that $\varphi_{m}(4r)\le2\varphi_{m}(r)$
    and $\mathfrak{R}_{m}(\mathcal{H}_{r})\le\varphi_{m}(r)$
    for all $r\ge r^{*}$,
    Then, for any $\delta\in(0,1]$,
    $ f_0 \in\mathcal{F} $ with
    $\Vert\ell^{\mathrm{hinge}}\circ f_{0}\Vert_{\infty}\le B_{0}$, and
    \[
      r>\max\left\{ 30\varphi_{m}(r),
      \frac{72V\ln(1/\delta)}{m},
      \frac{5B_{0}\ln(1/\delta)}{m},r^{*}\right\}\,,
    \]
    we have
    \[
      R_{\p,\lambda}^{1}(f_{m,N,\lambda})-R^{*}
      \le 6\left(R_{\p,\lambda}^h(f_{0})-R^{*}\right) + 3r
    \]
    with probability greater than $1-3\delta$.
\end{theorem}

To establish the fast rate of RFSVM using the theorem above, we must
understand the local Rademacher complexity of RFSVM: that is, find a formula
for $ \varphi_m(r) $. $ B_0,r^* $ and $ f_0 $ are only related with the
approximation error, and we leave the discussion of them to next sections.
The variance condition~\eqref{eq:variance} is satisfied under Assumption~\ref{massart}.
With this variance condition, we can upper bound the Rademacher complexity of RFSVM in
terms of number of features and regularization parameter.
It is particularly important to have $ 1/\lambda $ inside the logarithm function.

First, we will need the summation version of Dudley's inequality
using entropy number defined below, instead of covering number.
\begin{definition}
    For a semi-normed space $ (E,\Vert\cdot\Vert) $,
    we define its (dyadic) entropy number by
    \[
      e_n(E,\Vert\cdot\Vert)
      := \inf\left\{\varepsilon>0:\exists s_1,\ldots,s_{2^{n-1}} \in B^1
      \text{ s.t. } B^1\subset\bigcup_{i=1}^{2^{n-1}}B(s_i,\varepsilon)\right\}\,,
    \]
    where $B^1$ is the unit ball in $E$ and $B(a,r)$ is the ball with center at $a$
    and radius $r$.
\end{definition}
To take off the loss function from the hypothesis class, we have
the following lemma. $ \Vert\cdot\Vert_{L_{2}(D)} $ is the semi-norm defined by
$ \Vert\cdot\Vert_{L_{2}(D)} := (\frac{1}{m}\sum_i f^2(x_i))^{1/2} $.
\begin{lemma}
    $e_{i}(\mathcal{H}_{r},\Vert\cdot\Vert_{L_{2}(D)})
    \le e_{i}(\mathcal{F}_{r},\Vert\cdot\Vert_{L_{2}(D)})$
\end{lemma}
\begin{proof}
    Assume that $ T $ is an $\epsilon$-covering over $\mathcal{F}_{r}$ with
    $ |T| = 2^{i} $. By definition $ \epsilon \ge e_{i}(\mathcal{F}_{r},
    \Vert\cdot\Vert_{L_{2}(D)}) $.
    Then $ T' = \ell^1 \circ T - \ell^1 \circ f^*_{\p} $ is a covering over
    $ \mathcal{H}_{r} $. For any $ f $ and $ g $ in $ \mathcal{F}_{r} $,
    \begin{equation*}
      \Norm{\ell^{1} \circ f - \ell^{1} \circ g}_{L_{2}(D)}
      \le 1 \cdot \Norm{f - g}_{L_{2}(D)}\,,
    \end{equation*}
    because $ \ell^{1} $ is $ 1 $-Lipschitz. And hence the radius of the
    image of an $ \epsilon $-ball under $ \ell^{1} $ is less than $ \epsilon $.
    Therefore $ \ell^{1}\circ T-\ell^{1}\circ f_{\p}^{*} $ is an
    $ \epsilon $-covering over $ \mathcal{H}_{r} $ with cardinatily $ 2^{i} $
    and $ \epsilon \le e_{i}(\mathcal{F}_{r},\Vert\cdot\Vert_{L_{2}(D)}) $.
    By taking infimum over the radius of all such $ T $ and $ T' $,
    the statement is proved.
\end{proof}

Now we need to give an upper bound for the entropy number of $\mathcal{F}_{r}$
with semi-norm $\Vert\cdot\Vert_{L_{2}(D)}$ using a volumetric estimate.
\begin{lemma}
    $e_{i}(\mathcal{F}_{r},\Vert\cdot\Vert_{L_{2}(D)})\le3(2r/\lambda)^{1/2}2^{-i/2N}$.\end{lemma}
\begin{proof}
    Since $\mathcal{F}$ consists of functions
    \[
        f(x)=\frac{1}{\sqrt{N}}\sum_{i=1}^{N}w_{c_{i}}\cos\left(\frac{\omega_{i}\cdot x}{\gamma}\right)+w_{s_{i}}\sin\left(\frac{\omega_{i}\cdot x}{\gamma}\right)\,,
    \]
    under the semi-norm $ \Vert\cdot\Vert_{L_{2}(D)} $ it is isometric
    with the $ 2N $-dimensional subspace $ U $ of $ \r^{m} $ spanned by the vectors
    \[
        \left\{\left[\cos\left(\frac{\omega_{i}\cdot
        x_{1}}{\gamma}\right),\ldots,\cos\left(\frac{\omega_{i}\cdot
        x_{m}}{\gamma}\right)\right]^{\t},\left[\sin\left(\frac{\omega_{i}\cdot
        x_{1}}{\gamma}\right),\ldots,\sin\left(\frac{\omega_{i}\cdot
        x_{m}}{\gamma}\right)\right]^{\t}\right\}_{i=1}^{N}
    \]
    for fixed $m$ samples. For each $f\in\mathcal{F}_{r}$, we have $R_{\p,\lambda}^{1}(f)-R^{*}\le r\,,$
    which implies that $\Vert f\Vert_{\mathcal{F}}\le(2r/\lambda)^{1/2}\,.$
    By the property of RKHS, we get
    \[
        \left|f(x)\right|\le\Norm f_{\mathcal{F}}\Norm{k(x,\cdot)}_{\mathcal{F}}\le\left(\frac{2r}{\lambda}\right)^{1/2}\cdot1\,,
    \]
    where we use the fact that $ k(x,\cdot) $ is the evaluation functional in
    the RKHS.

    Denote the isomorphism from $\mathcal{F}$ (modulo the equivalent class under the
    semi-norm) to $U$ by $I$. Then we
    have
    \[
        I(\mathcal{F}_{r})\subset B_{\infty}^{m}\left(\left(\frac{2r}{m\lambda}\right)^{1/2}\right)\cap U\subset B_{2}^{m}\left(\left(\frac{2r}{\lambda}\right)^{1/2}\right)\cap U\,.
    \]
    The intersection region can be identified as a ball of radius $(2r/\lambda)^{1/2}$
    in $\r^{2N}$. Its entropy number by volumetric estimate is given
    by
    \[
        e_{i}\left(B_{2}^{2N}\left(\left(\frac{2r}{\lambda}\right)^{1/2}\right),\Vert\cdot\Vert_{2}\right)\le3\left(\frac{2r}{\lambda}\right)^{1/2}2^{-\frac{i}{2N}}\,.
    \]
\end{proof}

With the lemmas above, we can get an upper bound on the entropy number
of $\mathcal{H}_{r}$.
However, we should note that such an upper bound
is not the best when $i$ is small.
Because the ramp loss $\ell^{1}$ is bounded by $2$, the radius of $\mathcal{H}_{r}$
with respect to $\Vert\cdot\Vert_{L_{2}(D)}$ is bounded
by $1$, which is irrelevant with $r/\lambda$. This observation will
give us finer control on the Rademacher complexity.
\begin{lemma}
    \label{lem:emprad-RFSVM} Assume that $\lambda < 1/2$. Then
    \[
      \mathfrak{R}_{D}(\mathcal{H}_{r})
      \le \sqrt{\frac{(\ln16)N\log_2{1/\lambda}}{m}}\left(3\sqrt{2}\rho
      + 18\sqrt{r}\right)\,,
    \]
    where $\rho=\sup_{h\in\mathcal{H}_{r}}\Vert h\Vert_{L_{2}(D)}$.
\end{lemma}
\begin{proof}
    By Theorem 7.13 in \cite{Steinwart2008}, we have
    \begin{align*}
        \mathfrak{R}_{D}(\mathcal{H}_{r}) & \le  \sqrt{\frac{\ln16}{m}}\left(\sum_{i=1}^{\infty}2^{i/2}e_{2^{i}}(\mathcal{H}_{r}\cup\{0\},\Vert\cdot\Vert_{L_{2}(D)})+\sup_{h\in\mathcal{H}_{r}}\Vert h\Vert_{L_{2}(D)}\right)\,.
    \end{align*}
    It is easy to see that $e_{i}(\mathcal{H}_{r}\cup\{0\})\le e_{i-1}(\mathcal{H}_{r})$
    and $e_{0}(\mathcal{H}_{r})\le\sup_{h\in\mathcal{H}_{r}}\Vert h\Vert_{L_{2}(D)}$.
    Since $e_{i}(\mathcal{H}_{r})$ is a decreasing sequence with respect
    to $i$, together with the lemma above, we know that
    \[
        e_{i}(\mathcal{H}_{r})\le\min\left\{ \sup_{h\in\mathcal{H}_{r}}\Vert h\Vert_{L_{2}(D)},3\left(\frac{2r}{\lambda}\right)^{1/2}2^{-\frac{i}{2N}}\right\} \,.
    \]
    Even though the second one decays exponentially, it may be much greater
    than the first term when $2r/\lambda$ is huge for small $i$s. To
    achieve the balance between these two bounds, we use the first one
    for first $T$ terms in the sum and the second one for the tail. So
    \[
        \mathfrak{R}_{D}(\mathcal{H}_{r})\le\sqrt{\frac{\ln16}{m}}\left(\sup_{h\in\mathcal{H}_{r}}\Vert h\Vert_{L_{2}(D)}\sum_{i=0}^{T-1}2^{i/2}+3\left(\frac{2r}{\lambda}\right)^{1/2}\sum_{i=T}^{\infty}2^{i/2}2^{-\frac{2^{i}-1}{2N}}\right)\,.
    \]
    The first sum is $\frac{\sqrt{2}^{T}-1}{\sqrt{2}-1}$. When $ T $ is large enough,
    the second sum is upper bounded by the integral
    \begin{align*}
        \int_{T-1}^{\infty}2^{x/2}2^{-2^x-1/2N}\,\d x & \le
        \frac{6N}{2^{T/2}}\cdot2^{-\frac{2^{T}}{4N}}\,.
    \end{align*}
    To make the form simpler, we bound $\frac{\sqrt{2}^{T}-1}{\sqrt{2}-1}$
    by $3\cdot2^{T/2}$, and denote $\sup_{h\in\mathcal{H}_{r}}\Vert h\Vert_{L_{2}(D)}$
    by $\rho$. Taking $T$ to be
    \[
        \log_{2}\left(2N\log_{2}\left(\frac{1}{\lambda}\right)\right)\,,
    \]
    we get the upper bound of the form
    \[
        \mathfrak{R}_{D}(\mathcal{H}_{r})
        \le \sqrt{\frac{\ln16}{m}}\left(3\rho\sqrt{2N\log_2\frac{1}{\lambda}}
        + \frac{18\sqrt{Nr}}{\log_2(1/\lambda)}\right)\,,
    \]
    When $\lambda < 1/2$, $ \log_2{1/\lambda} > 1 $,
    so we can further enlarge the upper bound to the form
    \[
        \mathfrak{R}_{D}(\mathcal{H}_{r})
        \le \sqrt{\frac{(\ln16)N\log_2{1/\lambda}}{m}}\left(3\sqrt{2}\rho
        + 18\sqrt{r}\right)\,,
    \]
\end{proof}
Next lemma analyzes the expected Rademacher complexity for $\mathcal{H}_{r}$.
\begin{lemma}
    \label{lem:rad-RFSVM}
    Assume $ \lambda < 1/2 $ and $\e h^{2}(x,y) \le V\e h(x,y)$.
    Then
    \[
      \mathfrak{R}_{m}(\mathcal{H}_{r})
      \le  C_{1}\sqrt{\frac{N(V+1)\log_2(1/\lambda)}{m}}\sqrt{r}
      + C_2 \frac{N\log_2(1/\lambda)}{m}\,.
    \]
\end{lemma}
\begin{proof}
    With \lemref{lem:emprad-RFSVM}, we can directly compute the upper bound
    for $\mathfrak{R}_{m}(\mathcal{H}_{r})$ by taking expectation over
    $ D \sim \p^{m}$.
    \begin{align*}
        \mathfrak{R}_{m}(\mathcal{H}_{r})
        & =  \e_{D\sim\p^{m}}\mathfrak{R}_{D}(\mathcal{H}_{r})\\
        & \le \sqrt{\frac{(\ln16)N\log_2{1/\lambda}}{m}}
        \left(3\sqrt{2}\e\sup_{h\in\mathcal{H}_r}\Vert h \Vert_{L_2(D)}
        + 18\sqrt{r}\right)\,.
    \end{align*}
    By Jensen's inequality and A.8.5 in \cite{Steinwart2008}, we have
    \begin{align*}
        \e\sup_{h\in\mathcal{H}_{r}}\Vert h\Vert_{L_{2}(D)} & \le  \left(\e\sup_{h\in\mathcal{H}_{r}}\Vert h\Vert_{L_{2}(D)}^{2}\right)^{1/2}\\
         & \le  \left(\e\sup_{h\in\mathcal{H}_{r}}\frac{1}{m}\sum_{i=1}^{m}h^{2}(x_{i},y_{i})\right)^{1/2}\\
         & \le  \left(\sigma^{2}+8\mathfrak{R}_{m}(\mathcal{H}_{r})\right)^{1/2}\,,
    \end{align*}
    where $\sigma^{2}:=\e h^{2}$. When $\sigma^{2} > \mathfrak{R}_{m}(\mathcal{H}_{r})$,
    we have
    \begin{align*}
        \mathfrak{R}_{m}(\mathcal{H}_{r})
        & \le \sqrt{\frac{(\ln16)N\log_2(1/\lambda)}{m}}
        \left(9\sqrt{2}\sigma + 18\sqrt{r}\right) \\
        & \le \sqrt{\frac{(\ln16)N\log_2(1/\lambda)}{m}}
        \left(9\sqrt{2}\sqrt{Vr} + 18\sqrt{r}\right) \\
        & \le 36\sqrt{\frac{2(\ln16)N(V+1)\log_2(1/\lambda)}{m}}\sqrt{r}\,.
    \end{align*}
    The second inequality is because $ \e h^2 \le V\e h $ and
    $ \e h \le r $ for $ h\in\mathcal{H}_r $.

    When $\sigma^{2} \le \mathfrak{R}_{m}(\mathcal{H}_{r})$, we have
    \begin{align*}
        \mathfrak{R}_{m}(\mathcal{H}_{r})
        & \le \sqrt{\frac{(\ln16)N\log_2(1/\lambda)}{m}}
        \left(9\sqrt{2}\sqrt{\mathfrak{R}_m(\mathcal{H}_r)} + 18\sqrt{r}\right) \\
        & \le 36\sqrt{\frac{(\ln16)N\log_2(1/\lambda)}{m}}\sqrt{r}
        + 36^2\frac{(\ln16)N\log_2(1/\lambda)}{m}\,.
    \end{align*}
    The last inequality can be obtained by dividing the formula into two cases,
    either $ \mathfrak{R}_m(\mathcal{H}_r) < r $ or
    $ \mathfrak{R}_m(\mathcal{H}_r) \ge r $ and then take the sum of the upper
    bounds of two cases.

    Combining all these inequalities, we finally obtain an upper bound
    \[
      \mathfrak{R}_m(\mathcal{H}_r) \le C_1\sqrt{\frac{(V+1)N\log_2(1/\lambda)}{m}}\sqrt{r}
      + C_2\frac{N\log_2(1/\lambda)}{m}\,,
    \]
    where $ C_1 $ and $ C_2 $ are two absolute constants.
\end{proof}

The last lemma gives the explicit formula of $ \varphi_m(r) $. Now we can get
the formula for $ r $.
\begin{lemma}
  \label{lem:upper-bd-r}
  When
  \begin{equation}
    \label{eq:upper-bd-r}
    r = (900C_1^2 + 120C_2)N(V+1)\frac{\ln(1/\lambda)}{m}
    + (5B_0+72V)\frac{\ln(1/\delta)}{m}
  \end{equation}
  we have
  \begin{equation}
   r\ge \max\{30\varphi_{m}(r),
        \frac{72V\ln(1/\delta)}{m},
        \frac{5B_{0}\ln(1/\delta)}{m}\}.
  \end{equation}
\end{lemma}
It can be check by simply plugging $ r $ into $ \varphi_m(r) $.

\section{Proof of \thmref{thm:realizable}}
\label{app:realizable}

With \thmref{thm:est-error} and \lemref{lem:upper-bd-r}, we are almost done with
the proof of \thmref{thm:realizable}. The only missing part is an upper bound
of the approximation error $ R_{\p,\lambda}^h(f_{0})-R^{*} $. This upper bound
has been established in Proposition~1 in \cite{Bach2017}. We rephrase it as below.

\begin{theorem}
  \label{thm:app-err} (Proposition~1 of \cite{Bach2017})
  Assume that $ \phi $ is an optimized feature map and $ f $ belongs to the
  RKHS $ \mathcal{F} $ of $ \phi $. For $ \delta>0 $, when
  \begin{equation}
    N\ge 5d(\mu)\log\left(\frac{16d(\mu)}{\delta}\right)\,,
  \end{equation}
  there exists $ \beta\in\r^N $ with norm less than $ 2 $, such that
  \begin{equation}
    \sup_{\Vert f\Vert_\mathcal{F}\le 1} \Vert f-
    \beta\cdot\phi_N(\cdot)\Vert_{L^2(\mathcal{X},\p)} \le 2\sqrt{\mu}\,,
  \end{equation}
  with probability greater than $ 1-\delta $.
\end{theorem}

Now we prove two simple lemmas connecting the decay rate of $ \Sigma $ to the
magnitude of $ d(\mu) $.
\begin{lemma}
  \label{lem:dof}
  If $ \lambda_i(\Sigma) \le c_1 i^{-c_2} $, where $ c_2 > 1 $, we have
  \begin{equation}
    d(\mu) \le \frac{2c_2}{c_2-1}\left(\frac{c_1}{\mu}\right)^{1/c_2}\,,
  \end{equation}
  for $ \mu < c_1 $.

  If $ \lambda_i(\Sigma) \le c_3 \exp(-c_4 i^{1/d})$, we have
  \begin{equation}
    d(\mu) \le 5c_4^{-d}\ln^d(c_3/\mu)\,,
  \end{equation}
  for $ \mu < c_3\exp\left(-\left(c_4\vee\frac{1}{c_4}\right)d^2\right) $.
\end{lemma}
\begin{proof}
  Both results make use the following observation:
  \begin{equation}
    d(\mu) = \sum_{i=1}^\infty \frac{\lambda_i}{\lambda_i + \mu}
    \le m_\mu + \frac{1}{\mu}\sum_{m_\mu + 1}^\infty \lambda_i\,,
  \end{equation}
  where $ m_\mu = \max\{ i:\lambda_i \le \mu \}\vert $.

  When $ \lambda_i\le c_1i^{-c_2} $,
  denote $ t_\mu = (c_1/\mu)^{1/c_2} $
  and then $ m_\mu = \lfloor t_\mu \rfloor $. For the tail part,
  \begin{align}
    \frac{1}{\mu}\sum_{m_\mu + 1}^\infty \lambda_i & \le 1 + \frac{1}{\mu}
    \int_{t_\mu}^\infty c_1 x^{-c_2}~\d x \\
    & \le 1 + \frac{1}{c_2-1}\left(\frac{c_1}{\mu}\right)^{\frac{1}{c_2}}\,.
  \end{align}
  Combining them together, when $ c_1/\mu>1 $, the constant $ 1 $ can be absorbed
  by the second term with a coefficient $ 2 $.

  When $ \lambda_i \le c_3\exp(-c_4i^{1/d}) $,
  denote $ t_\mu = \frac{1}{c_4^d}\ln^d\left(\frac{c_3}{\mu}\right) $, and then
  $ m_\mu = \lfloor t_\mu \rfloor $. For the tail part, we need to discuss different
  situations.

  First, if $ d = 1 $, then we directly have
  \begin{align}
    \frac{1}{\mu}\sum_{m_\mu + 1}^\infty \frac{\lambda_i}{\lambda_i + \mu}
    & \le \frac{1}{\mu}\left(\mu + \int_{t_\mu}^\infty c_3\exp(-c_4 x)~\d x\right) \\
    & = 1 + \frac{1}{c_4}\,.
  \end{align}
  When $ \mu < c_3\exp(-(c_4\vee\frac{1}{c_4})) $, we can combine these terms into
  $ 3t_\mu $.

  Second, if $ d \ge 2 $, when $ \mu \le c_3\exp(-c_4 e) $, we have that
  \begin{equation}
    \exp(-c_4 x^{1/d}) \le \exp(-c_4\frac{t_\mu^{1/d}}{\ln t_\mu}\ln x)
    = x^{-c_4\frac{t_\mu^{1/d}}{\ln t_\mu}}\,.
  \end{equation}
  Then,
  \begin{align}
    \frac{1}{\mu}\sum_{m_\mu + 1}^\infty \lambda_i & \le 1 + \frac{1}{\mu}\int_{t_\mu}^\infty
    c_3 \exp(-c_4 x^{-1/d})~\d x \\
    & \le 1 + \frac{c_3}{\mu}\int_{t_\mu}^\infty x^{-c_4\frac{t_\mu^{1/d}}{\ln t_\mu}} \\
    & = 1 + \frac{t_\mu}{c_4\frac{t_\mu^{1/d}}{\ln t_\mu} - 1}\,.
  \end{align}
  When $ c_4 \ge 1 $, we may assume that $ \mu \le c_3\exp(-c_4 d^2) $, and then
  \begin{equation}
    c_4\frac{t_\mu^{1/d}}{\ln t_\mu} - 1 \ge \frac{c_4 d^2}{2d\ln d} \ge \frac{4}{3}\,.
  \end{equation}
  So the upper bound has the form $ 5t_\mu $.

  When $ c_4 < 1 $, we may assume that $ \mu \le c_3\exp(-d^2/c_4) $, and then
  \begin{equation}
    c_4\frac{t_\mu^{1/d}}{\ln t_\mu} - 1 \ge \frac{d^2/c_4}{2d\ln(d/c_4)} \ge \frac{4}{3}\,.
  \end{equation}
  So the upper bound also has the form $ 5t_\mu $.
\end{proof}

Now with all these preparation, we can complete our proof of \thmref{thm:realizable}
\begin{proof}
  Under the assumption of \thmref{thm:realizable}, $ B_0 = 1 $ and $ r^* =0 $
  in \thmref{thm:est-error}. By \lemref{lem:upper-bd-r}, we have
  \begin{equation}
    r = (900C_1^2 + 120C_2)N(V+1)\frac{\ln(1/\lambda)}{m}
    + (5+72V)\frac{\ln(1/\delta)}{m}\,.
  \end{equation}
  By \thmref{thm:app-err}, we have
  \begin{equation}
    R^h_{\p,\lambda}(f_0) - R^* \le 2\sqrt{\mu}R + 4R^2\frac{\lambda}{2}\,,
  \end{equation}
  with probability $ 1-\delta $ when $ N\ge 5d(\mu)\log\left(\frac{16d(\mu)}{\delta}\right) $.

  When the spectrum of $ \Sigma $ decays polynomially,
  \begin{equation}
    d(\mu) \le \frac{2c_2}{c_2-1}\left(\frac{c_1}{\mu}\right)^{1/c_2}\,.
  \end{equation}
  Assume $ m>c_1^{-(2+c_2)/(2c_2)} $.
  By choosing $ \mu=c_1m^{-\frac{2c_2}{2+c_2}}<c_1 $ and $ \lambda=m^{-c_2/(2+c_2)} $, we have
  \begin{equation}
    N = 10c_{1,2}m^{\frac{2}{2+c_2}}(\ln(32c_{1,2}m^{\frac{2}{2+c_2}})+\ln(1/\delta))\,,
  \end{equation}
  and
  \begin{align}
    R^h_{\p,\lambda}(f_{m,N,\lambda}) - R^* & \le \frac{12R}{m^{\frac{c_2}{2+c_2}}}
    + \frac{12R^2}{m^{\frac{c_2}{2+c_2}}}\\
    & + 30 C_{1,2} c_{1,2}(\ln32c_{1,2}+\frac{2}{2+c_2}\ln m+\ln(1/\delta))
    (V+1)\frac{c_2}{2+c_2}\frac{\ln m}{m^{\frac{c_2}{2+c_2}}} \\
    & + \frac{15+216V}{m}\ln(1/\delta)\,,
  \end{align}
  with probability $ 1-4\delta $,
  where
  \begin{equation}
    C_{1,2} = 900C_1^2+120C_2,\quad c_{1,2}=\frac{c_2 c_1^{1/c_2}}{c_2-1}\,.
  \end{equation}

  When the spectrum of $ \Sigma $ decays sub-exponentially,
  \begin{equation}
    d(\mu) \le 5c_4^{-d}\ln^d(c_3/\mu)\,.
  \end{equation}
  Assume that $ m>\exp(-(c_4\vee\frac{1}{c_4})d^2/2) $.
  By choosing $ \mu= c_3/m^2 $ and $ \lambda = 1/m $, we have
  \begin{equation}
    N = 25c_{d,4}\ln^d(m)(\ln(80c_{d,4}\ln^d(m))+\ln(1/\delta))\,,
  \end{equation}
  and
  \begin{align}
    R^h_{\p,\lambda}(f_{m,N,\lambda}) - R^* & \le \frac{12R\sqrt{c_3}}{m} \\
    & + \frac{12R^2}{m} + 150 C_{1,2} c_{d,4}(\ln160 c_{d,4}+d\ln\ln m+\ln(1/\delta))
    (V+1)\frac{\ln^{d+1} m}{m} \\
    & + \frac{15+216V}{m}\ln(1/\delta)\,,
  \end{align}
  with probability $ 1-4\delta $,
  where
  \begin{equation}
    C_{1,2} = 900C_1^2+120C_2,\quad c_{d,4}=\left(\frac{2}{c_4}\right)^d\,.
  \end{equation}
\end{proof}

\section{Proof of \thmref{thm:unrealizable}}
\label{app:unrealizable}

\thmref{thm:unrealizable} requires a further analysis of the approximation
error of RKHS to the Bayes classifier. This part adopts \cite{Steinwart2008}'s
idea of margin noise exponent. We say that the data distribution $ \p $ has margin
noise exponent $ \beta>0 $ if there exists a positive constant $ c $ such that
\begin{equation}
        \int_{\{x:\Delta(x)<t\}} \vert y\vert\d\p(x,y) \le ct^{-\beta}\quad\forall t\in(0,1)\,.
\end{equation}
Therefore, infinite $ \beta $ corresponds to our separation condition with $ \tau=1 $.
However, the original proof of the approximation error that works with the
margin noise exponent cannot be generalized to the case of infinite $ \beta $,
because the coefficient $ \Gamma(d+\beta)/2^d $ will blow up (see Theorem
8.18 in \cite{Steinwart2008}). This issue can be resolved by modifying the
original proof, as shown below.

\begin{lemma}
  \label{lem:app-separation}
  Assume that there exists $ \tau>0 $ such that
  \begin{equation}
    \int_{\{x: \Delta(x)<t \}} \vert
    2\eta(x)-1\vert~\mathrm{d}\mathbb{P}_\mathcal{X}(x) = 0\,,
    \forall t<\tau\,,
  \end{equation}
  where $ \mathcal{X}\subset B^d(\rho) $ and $ \eta(x) $ is a version of $ \mathbb{P}(y=1|x) $. Then there exists a function $ f $ in the RKHS generated by the kernel
  \begin{equation}
    k_\gamma(x,x')=\exp\left(-\frac{\Vert x-x' \Vert^2}
    {2\gamma^2}\right)
  \end{equation}
  where $ \gamma<\tau/\sqrt{d-1} $ such that
  \begin{align*}
    R^h(f) - R^* & < \frac{4\tau^{d-2}}{\Gamma(d/2)}\exp\left(-\frac{\tau^2}{\gamma^2}\right)\gamma^{d-2}\,,
    \\
    \Vert f \Vert_\mathcal{F} & \le \frac{(\sqrt{\pi/2}\rho^2)^{d/2}}{\Gamma(d/2+1)}\gamma^{-d/2}
  \end{align*}
  and
  \begin{equation}
    \vert f(x)\vert \le 1\,.
  \end{equation}
\end{lemma}

\begin{proof}
First we define
\begin{equation}
  \mathcal{X}_y := \left\{ x: (2\eta(x)-1)y > 0 \right\}\text{ for }y=\pm 1\,,
\end{equation}
and $ g(x):=(\sqrt{2\pi}\gamma)^{-d/2}\mathrm{sign}(2\eta(x)-1) $.
It is square integrable since $ \eta(x)=1/2 $ for all $ x\notin\mathcal{X} $.
Then we map $ g $ onto the RKHS by the integral operator determined by $ k_\gamma $,
\begin{equation}
  f(x) := \int_{\mathbb{R}^d} \phi_\gamma(t;x)g(t)~\mathrm{d}t\,,
\end{equation}
where
\begin{equation}
  \phi_\gamma(t;x) =
  \left(\frac{2}{\pi\gamma^2}\right)^{d/4}
  \exp\left(-\frac{\Vert x-t\Vert^2}{\gamma^2}\right) \,.
\end{equation}
Note that it is a special property of Gaussian kernel that the feature map onto
$ L^2(\mathbb{R}^d) $ also has a Gaussian form. For other type of kernels,
we may not have such a convenient characterization.

We know that
\begin{equation}
  \Vert f\Vert_\mathcal{H} = \Vert g\Vert_{L^2}
  \le \frac{\sqrt{\mathrm{Vol}(B^d(\rho))}}{(\sqrt{2\pi}\gamma)^{d/2}}
  = \frac{(\sqrt{\pi/2}\rho^2)^{d/2}}{\Gamma(d/2+1)}\gamma^{-d/2}\,.
\end{equation}
Moreoever,
\begin{align*}
  \vert f(x)\vert & \le \int_{\mathbb{R}^d}\phi_\gamma(t;x)(\sqrt{2\pi}\gamma)^{-d/2}~\mathrm{d}t \\
  & = (\pi\gamma^2)^{-d/2}\int_{\mathbb{R}^d}\exp\left(-\frac{\Vert x-t\Vert^2}{\gamma^2}\right)~\mathrm{d}t \\
  & = 1\,.
\end{align*}
Since $ f $ is uniformly bounded by $ 1 $, by Zhang's inequality, we have
\begin{equation}
  R^h(f) - R^* = \mathbb{E}_{\mathbb{P}_\mathcal{X}}
  (\vert f(x)-\mathrm{sign}(2\eta(x)-1)\vert\vert2\eta(x)-1\vert)\,.
\end{equation}
Now we give an upper bound on $ \vert f(x)-\mathrm{sign}(2\eta(x)-1)\vert $.
Assume $ x\in\mathcal{X}_1 $. Then we know that $ f(x)\le\mathrm{sign}(2\eta(x)-1)=1 $,
\begin{align*}
  1 - f(x) & = 1 - \left(\frac{1}{\pi\gamma^2}\right)^{d/2}
  \int_{\mathbb{R}^d}\exp\left(-\frac{\Vert x-t\Vert^2}{\gamma^2}\right)
  \mathrm{sign}(2\eta(t)-1)~\mathrm{d}t \\
  & = 1 - \left(\frac{1}{\pi\gamma^2}\right)^{d/2}
  \int_{\mathcal{X}_1}\exp\left(-\frac{\Vert x-t\Vert^2}{\gamma^2}\right)~\mathrm{d}t \\
  & + \left(\frac{1}{\pi\gamma^2}\right)^{d/2}
  \int_{\mathcal{X}_{-1}}\exp\left(-\frac{\Vert x-t\Vert^2}{\gamma^2}\right)~\mathrm{d}t \\
  & \le 2 - 2\left(\frac{1}{\pi\gamma^2}\right)^{d/2}
  \int_{B(x,\Delta(x))}\exp\left(-\frac{\Vert x-t\Vert^2}{\gamma^2}\right)~\mathrm{d}t \\
  & \le 2 - 2\left(\frac{1}{\pi\gamma^2}\right)^{d/2}
  \int_{B(0,\Delta(x))}\exp\left(-\frac{\Vert t\Vert^2}{\gamma^2}\right)~\mathrm{d}t \\
  & = 2\left(\frac{1}{\pi\gamma^2}\right)^{d/2}
  \int_{\mathbb{R}^d\backslash B(0,\Delta(x))}\exp
  \left(-\frac{\Vert t\Vert^2}{\gamma^2}\right)~\mathrm{d}t \\
  & = \frac{4}{\Gamma(d/2)\gamma^d}
  \int_{\Delta(x)}^\infty\exp\left(-\frac{ r^2}{\gamma^2}\right)r^{d-1}~\mathrm{d}r \,.
\end{align*}
Here the key is that $ B(x,\Delta(x))\subset\mathcal{X}_1 $
when $ x\in\mathcal{X}_1 $. For $ x\in\mathcal{X}_{-1} $,
we have the same upper bound for $ 1+f(x) $. Therefore, we have
\begin{align*}
  R^h(f)-R^* & \le \frac{4}{\Gamma(d/2)\gamma^d}\int_\mathcal{X}\int_{0}^\infty
  \mathbf{1}_{(\Delta(x),\infty)}(r)\exp\left(-\frac{r^2}{\gamma^2}\right)r^{d-1}
  \vert2\eta(x)-1\vert~\mathrm{d}r\mathrm{d}\mathbb{P}_\mathcal{X}(x) \\
  & = \frac{4}{\Gamma(d/2)\gamma^d}\int_{0}^\infty\int_\mathcal{X}\mathbf{1}_{(0,r)}
  (\Delta(x))\exp\left(-\frac{r^2}{\gamma^2}\right)r^{d-1}
  \vert2\eta(x)-1\vert~\mathrm{d}\mathbb{P}_\mathcal{X}(x)\mathrm{d}r\\
  & \le \frac{4}{\Gamma(d/2)\gamma^d}\int_{\tau}^\infty\exp
  \left(-\frac{r^2}{\gamma^2}\right)r^{d-1}~\mathrm{d}r\\
\end{align*}
To get the last line, we apply the assumption on the expected label clarity.
Now we only need to give an estimate of the integral.
\begin{equation}
  \int_\tau^\infty\exp\left(-\frac{r^2}{\gamma^2}\right)r^{d-1}~\mathrm{d}r
  \le \int_{\tau}^\infty C\exp\left(-\alpha\frac{r^2}{\gamma^2}\right)~\mathrm{d}r
\end{equation}
where
\begin{equation}
  C = \tau^{d-1}\exp(-(d-1)/2)\quad\alpha = 1-2\gamma^2\tau^{-2}(d-1)\,.
\end{equation}
It is required that $ \gamma < \sqrt{2}\tau/\sqrt{d-1} $ so that $ \alpha>0 $.
And then we can give an upper bound to the excess risk
\begin{equation}
  R^h(f)-R^* \le \frac{4\tau^d}{\Gamma(d/2)(2\tau^2-(d-1)\gamma^2)}
  \exp\left(-\frac{\tau^2}{\gamma^2}\right)\gamma^{d-2}\,.
\end{equation}
If we further require that $ \gamma<\tau/\sqrt{d-1} $, then we have a simpler upper bound,
\begin{equation}
  \frac{4\tau^{d-2}}{\Gamma(d/2)}\exp\left(-\frac{\tau^2}{\gamma^2}\right)\gamma^{d-2}\,.
\end{equation}

\end{proof}

Some remarks on this result:
\begin{enumerate}
\item The proof follows almost step by step the proof of \cite{Steinwart2008}.
The only difference occurs at where we apply our assumption.
\item The approximation error is basically dominated by $ \exp(-c/\gamma^2) $,
and thus leaves us large room for balancing with the norm of the approximator.
\item The proof here only works for Gaussian kernel. A similar conclusion
may hold for General RBF kernels using the fact that any RBF kernel can be
expressed as an average of Gaussian kernel over different values of $ \gamma $.
A relevant reference is \cite{Scovel2010}.
\end{enumerate}

The last component for the proof of \thmref{thm:unrealizable} is the sub-exponential
decay rate of the spectrum of $ \Sigma $ determined by the Gaussian kernel.
The distribution of the spectrum of the convolution operator with respect to a
distribution density function $ p $ has been studied by \cite{Widom1963}. It
shows that the number of eigenvalues of $ \Sigma $ greater than $ \mu $
is asymptotic to $ (2\pi)^{-d} $ times the volume of
\begin{equation*}
  \left\{ (x,\xi):p(x)\hat{k}(\xi) > \mu \right\}\,,
\end{equation*}
where $ \hat{k} $ is the Fourier transform of the kernel function $ k $.
By applying \cite{Widom1963}'s work in our case, we have the following lemma.
It is essentially Corollary 27 in \cite{Harchaoui2008}, but our version explicitly
shows the dependence on the band width $ \beta $.
\begin{lemma}
  \label{lem:spectrum}
  Assume $ \hat{k}(\xi)\le \alpha\exp(-\beta\Vert\xi \Vert^2) $.
  If the density function $ p(x) $ of probability distribution
  $ \mathbb{P_\mathcal{X}} $ is bounded by $ B $ and $ \mathcal{X} $ is a bounded
  subset of $ \r^d $ with radius $ \rho $, then
  \begin{equation*}
    \lambda_i(\Sigma) \le C\alpha B\exp\left(-\beta\left(\frac{4\Gamma^{4/d}(d/2+1)}{\pi^{4/d\rho^2}}\right)
    i^{2/d}\right)\,,
  \end{equation*}
  where $ \lambda_1\ge\lambda_2\ge\cdots $ are
  eigenvalues of $ \Sigma $ in descending order.
\end{lemma}
\begin{proof}
  Denote by $ E_t $ the set
  \[
    \left\{ (x,\xi):\hat{k}(\xi)p(x)>t \right\}\,.
  \]
  The volume, that is, the Lebesgue measure of $ E_t $ is denoted by $ \mathrm{Vol}(E_t) $.
  By Theorem II of \cite{Widom1963}, the non-increasing function $ \phi(\alpha) $
  defined on $ \r^+ $ which is equi-measurable with $ p(x)\hat{k}(\xi) $ describes
  the behaviour of $ \lambda_i $s. Indeed, $ \lambda_i\le C\phi((2\pi)^d i) $.
  By the volume formula of $ 2d $-dimensional ball we have the following estimate,
  \begin{align*}
    \sup \{ s\in\r^+:\phi(s)>t \} & = \mathrm{Vol}(E_t) \\
    & \le C_{d,\rho}\left(\frac{\ln(\alpha B/t)}{\beta}\right)^{d/2}\,,
  \end{align*}
  where
  \begin{equation}
    C_{d,\rho} = \frac{\rho^d\pi^{d+2}}{\Gamma^2(d/2+1)}\,.
  \end{equation}
  Solving for $ t $, we know that
  \[
    \phi(s) \le
    \alpha B\exp\left(-\beta\left(\frac{s}{A}\right)^{2/d}\right)\,.
  \]
  Therefore, we have
  \begin{align}
    \lambda_i(\Sigma) & \le
    C\alpha B\exp\left(-\beta\left(\frac{(2\pi)^d i}{A}\right)^{2/d}\right) \\
    & = C\alpha B\exp\left(-\beta\left(\frac{4\Gamma^{4/d}(d/2+1)}{\pi^{4/d}\rho^2}\right)
    i^{2/d}\right)\,.
  \end{align}
\end{proof}

Now we can prove \thmref{thm:unrealizable}.
\begin{proof}
  Note that, by \lemref{lem:app-separation}, we can construct $ g\in\mathcal{F} $ such
  that $ R^h_{\p,\lambda}-R^* $ is controlled. And by \thmref{thm:app-err},
  we can find an $ f_0\in\mathcal{F}_N $ with similar risk to $ g $.
  And this will be our $ f_0 $ as required by \thmref{thm:est-error}. So we have
  \begin{align}
    R^h_{\p,\lambda}(f_0)-R^* & \le \frac{2(\sqrt{\pi/2}\rho^2)^d}{\Gamma^2(d/2+1)}
    \frac{\lambda}{\gamma^d} + \frac{2(\sqrt{\pi/2}\rho^2)^{d/2}}{\Gamma(d/2+1)}\sqrt{\mu} \\
    & + \frac{4\tau^{d-2}}{\Gamma(d/2)}\exp\left(\frac{\tau^2}{\gamma^2}\right)\gamma^{d-2}\,,
  \end{align}
  and $ \Vert f_0\Vert_{\mathcal{F}_N}\le 2 $,
  with probability $ 1-\delta $, when $ N = 5d(\mu)\ln(16d(\mu)/\delta) $.
  We choose $ \gamma = \tau/\sqrt{\ln m} $ and $ \lambda = 1/m $.
  Under the boundedness assumption on the density function and the property of
  Gaussian kernel, we know that by \lemref{lem:spectrum},
  \begin{equation}
    \lambda_i(\Sigma) \le C\gamma B\exp\left(-\gamma^2
    \frac{4\Gamma^{4/d}(d/2+1)}{\pi^{4/d}\rho^2}
    i^{2/d}\right)\,.
  \end{equation}
  And similar to the second part of \thmref{thm:realizable}, by identifying
  \begin{equation}
    c_3=C\gamma B = CB\tau/\sqrt{\ln m}\quad
    c_4=\frac{4\tau^2\Gamma^{4/d}(d/2+1)}{\pi^{4/d}\rho^2\ln m}:=\frac{A}{\ln m}\,,
  \end{equation}
  and choosing $ \mu = c_3/(m^{2d^2}\vee \exp(\frac{d^2}{c_4}\vee c_4d^2)) $, we have
  \begin{equation}
    d(\mu) \le 5d^{2d}(c_4^{-2d}\vee 1\vee c_4^{-d}2^d\ln^d m)\,.
  \end{equation}
  Then when $ m\ge \exp(A) $, we have $ d(\mu)\le 5(A^2\wedge A/2)^{-d}\ln^{2d} m $,
  and
  \begin{align}
    N & = 5d(\mu)(\ln(16d(\mu))+\ln(1/\delta)) \\
    & \le 25(A^2\wedge A/2)^{-d}\ln^{2d}m(\ln(80(A^2\wedge A/2)^{-d})+2d\ln\ln m+\ln(1/\delta))\,.
  \end{align}
  Plug $ N $ and $ \lambda $ into \eqref{eq:upper-bd-r}.
  \begin{align}
    3r & = 75 C_{1,2} c_{d,\tau,\rho}(\ln(80c_{d,\tau,\rho})+2d\ln\ln m+\ln(1/\delta))
    (V+1)\frac{\ln^{2d+1} m}{m} \\
    & + \frac{15+216V}{m}\ln(1/\delta) + 3r^*\,,
  \end{align}
  where
  \begin{equation}
    C_{1,2} = 900C_1^2+120C_2,\quad c_{d}=(A^2\wedge A/2)^{-d}\,.
  \end{equation}
  We can bound $ r^* $ by $ R^h_{\p,\lambda}(f_0)-R^* $. Therefore, the overall
  upper bound on the excess error is
  \begin{align}
    R^1_{\p,\lambda}(f_{m,N,\lambda})-R^* & \le \frac{18(\sqrt{\pi/2}\rho^2)^d}{\Gamma^2(d/2+1)}
    \frac{\ln^{d/2}m}{\tau m} + \frac{18(\sqrt{\pi/2}\rho^2)^{d/2}}{\Gamma(d/2+1)}
    \frac{\sqrt{CB\tau}\ln^{1/4}m}{m^{d^2}} \\
    & + \frac{36\tau^{d-2}}{\Gamma(d/2)}\frac{\tau^{d-2}}{m\ln^{d/2-1}m} \\
    & + 75 C_{1,2} c_{d,\tau,\rho}(\ln(80c_{d,\tau,\rho})+2d\ln\ln m+\ln(1/\delta))
    (V+1)\frac{\ln^{2d+1} m}{m} \\
    & + \frac{15+216V}{m}\ln(1/\delta)\,.
  \end{align}
\end{proof}

\section{Learning Rate without Optimized Feature Maps}
\label{app:unif}

In this section, we discuss the learning rate of RFSVM without an optimized
feature map. As shown by \cite{Rudi2017}, RFKRR can achieve excess risk of
$ O(1/\sqrt{m}) $ using $ O(\sqrt{m}\log(m)) $ features. However, it is
inappropriate to directly compare this result with the learning rate in
classification scenario. Because as surrogate loss functions, least square
loss has a different calibration function with for example hinge loss. Basically,
$ O(\epsilon) $ risk under square loss only implies $ O(\sqrt{\epsilon}) $ risk under
$ \mathrm{0-1} $ loss, while $ O(\epsilon) $ risk under hinge loss implies
$ O(\epsilon) $ risk under $ \mathrm{0-1} $ loss. Therefore, \cite{Rudi2017}'s
analysis only implies an excess risk of $ O(m^{-1/4}) $ in classification
problems with $ \tilde{O}(\sqrt{m}) $ features.

For RFSVM, we expect a similar result. Without assuming an optimized feature map,
the leverage score can only be upper bounded by $ \kappa^2/\mu $, where $ \kappa $
is the upper bound on the function $ \phi(\omega;x) $ for all $ \omega,x $.
Substituting $ \kappa^2/\mu $ for $ d(\mu) $ in the proofs of learning rates,
we need to balance $ \sqrt{\mu} $ with $ 1/(\mu m) $ to achieve the optimal rate.
This balance is not affected by the spectrum of $ \Sigma $ or whether $ f^*_\p $
belongs to $ \mathcal{F} $. Obviously, setting $ \mu=m^{-2/3} $, we get a learning
rate of $ m^{-1/3} $, with $ \tilde{O}(m^{2/3}) $ random features. Even though
this result is also new for RFSVM in regularized formulation, the gap to
previous analysis like \cite{Rahimi2008} is too large. Considering that
the random features used in practice that are not optimized also have quite
good performance, we need further analysis on RFSVM without optimized feature
map.

\newpage{}
\section{Supplementary Figures}
\label{app:figures}

\begin{figure}[hb]
    \begin{tabular}{cc}
        \includegraphics[scale=0.45]{opt_resultsE3N1.eps} &
        \includegraphics[scale=0.45]{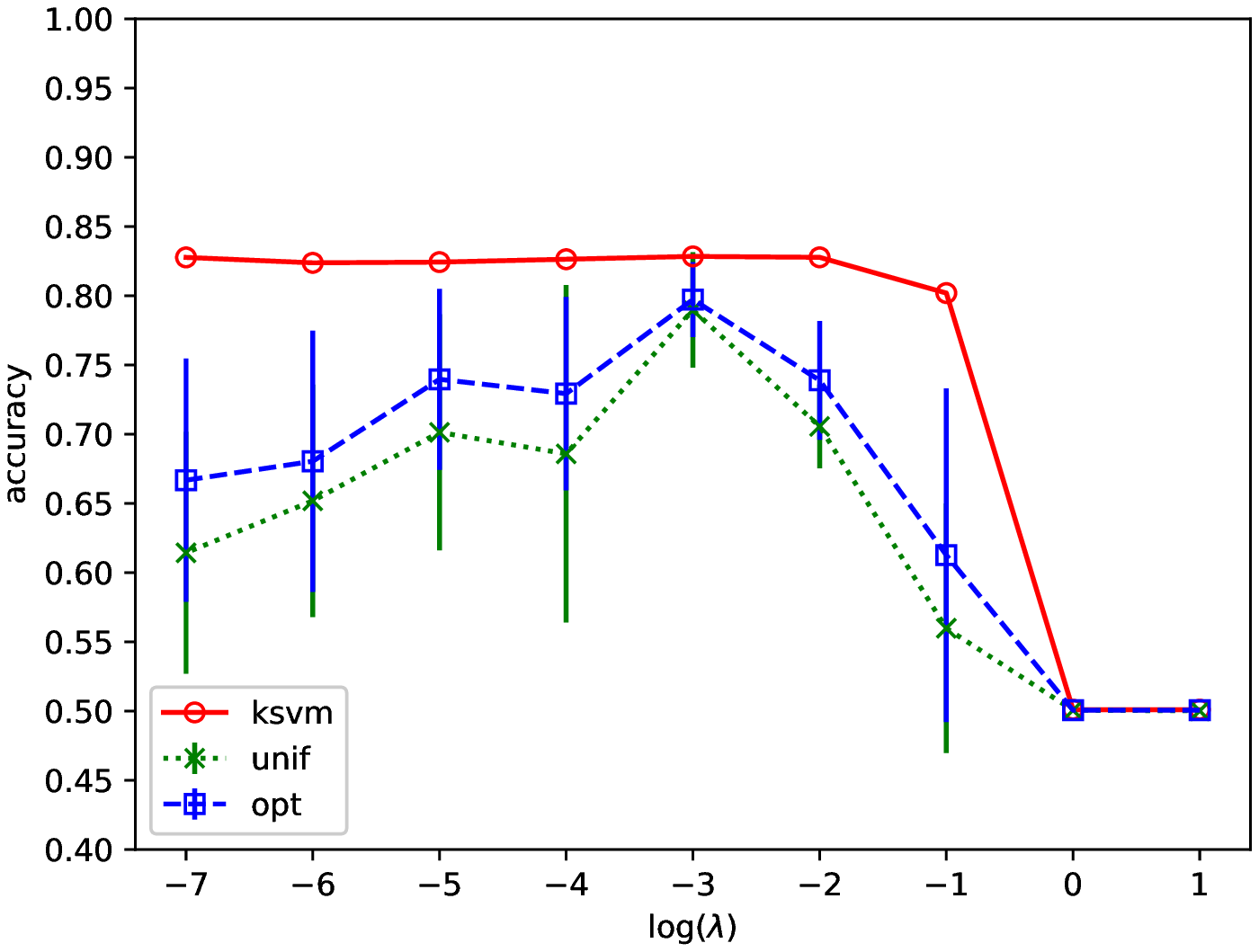} \\
        $N=1$ & $N=3$ \\
        \includegraphics[scale=0.45]{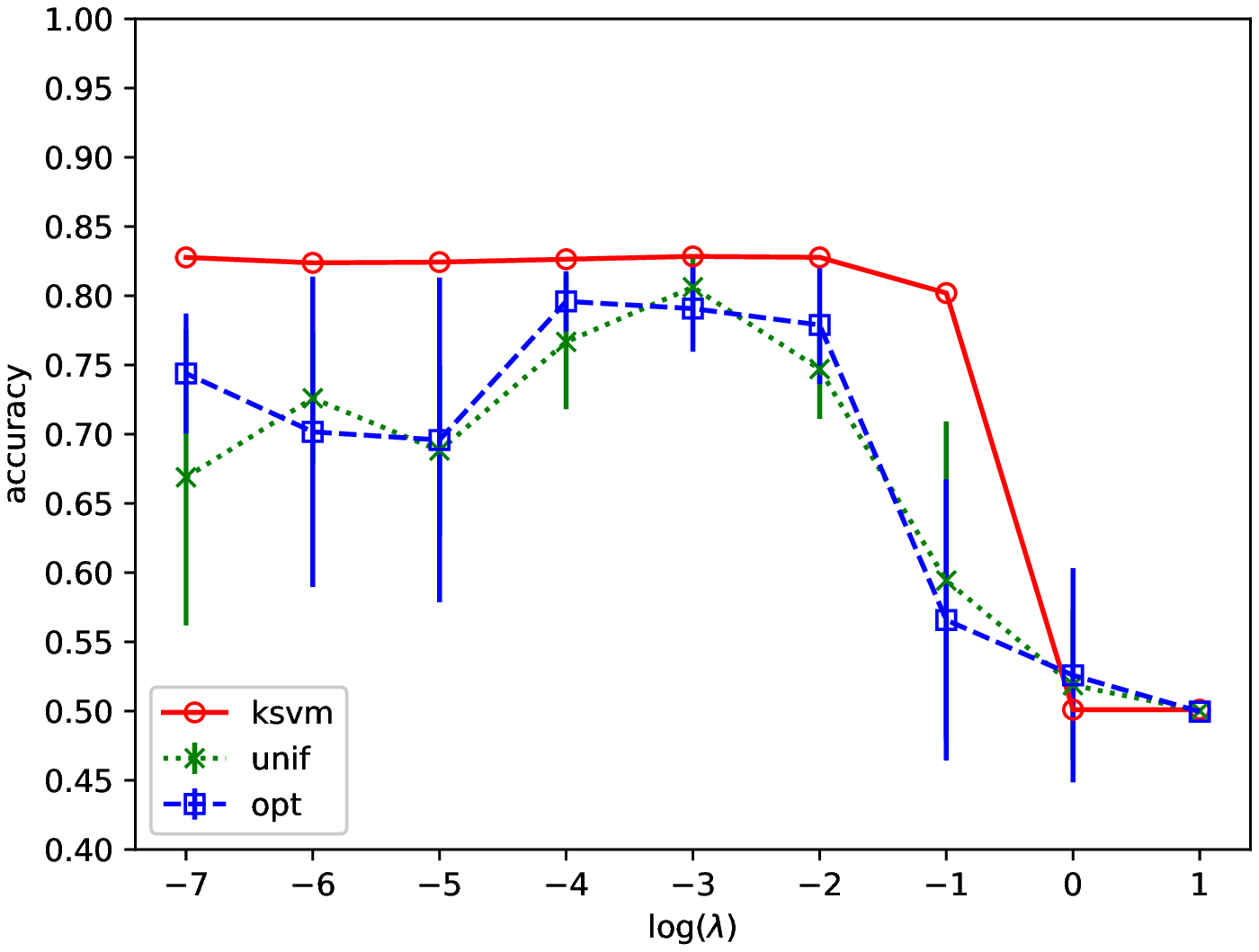} &
        \includegraphics[scale=0.45]{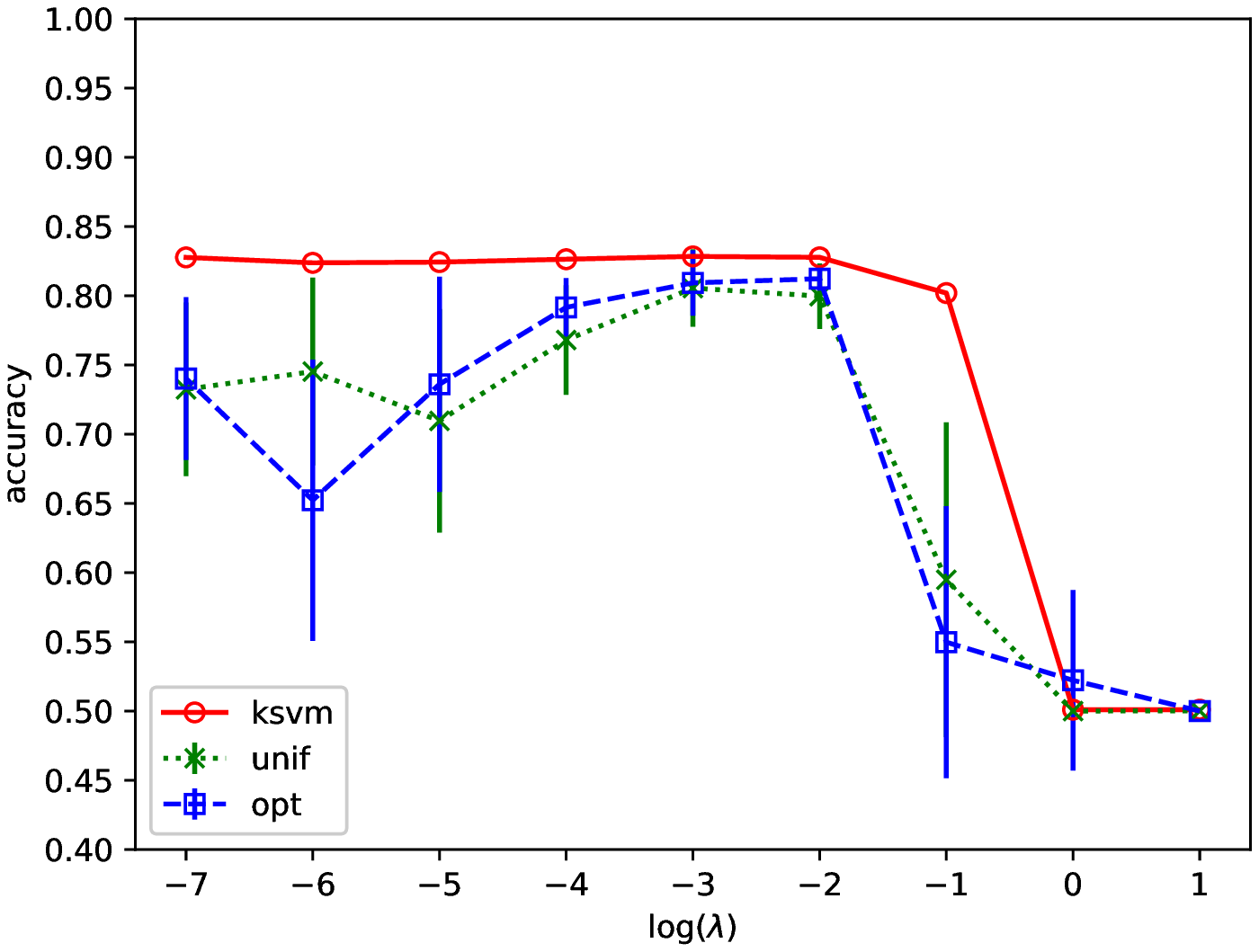} \\
        $N=5$ & $N=10$ \\
    \end{tabular}
    \caption{Comparison between RFSVMs with KSVM Using Gaussian Kernel.}
    \medskip
    \small
    ``ksvm'' is for KSVM with Gaussian kernel,
    ``unif'' is for RFSVM with direct feature sampling,
    and ``opt'' is for RFSVM with reweighted feature sampling.
    Error bars represent standard deviation over 10 runs.
    Each sub-figure shows the performance of RFSVM with different number
    of features $N$.
\end{figure}

\begin{figure}
  \includegraphics[width=\textwidth]{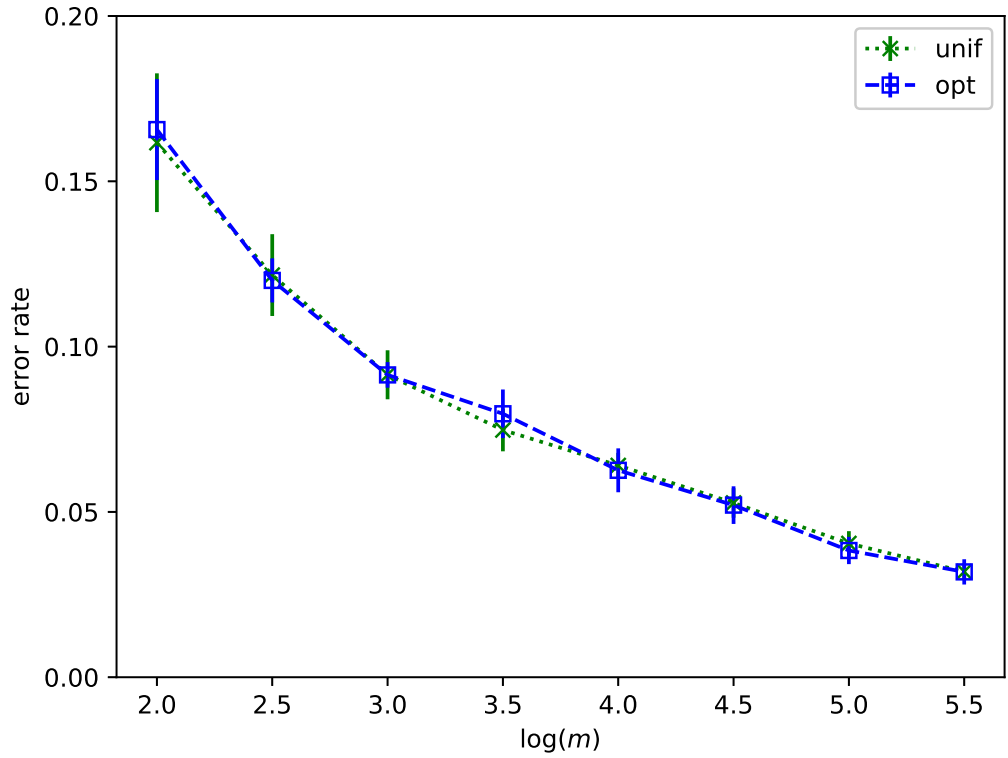}
  \caption{
      The excess risks of RFSVMs
      with the simple random feature selection (``unif'')
      and the reweighted feature selection (``opt'') are shown for different sample
      sizes in the binary classification task over 10 dimensional data. The
      data with probability $ 0.9 $ to be -1 are within the 10 dimensional
      ball centered at the origin and radius $ 0.9 $, and the data with
      probability 0.9 to be 1 are within the shell of radius 1.1 to 2.
      The error rate is the excess risk. The error bars represent the
      standard deviation over 10 runs.
      }
\end{figure}

\begin{figure}
  \includegraphics[width=\textwidth]{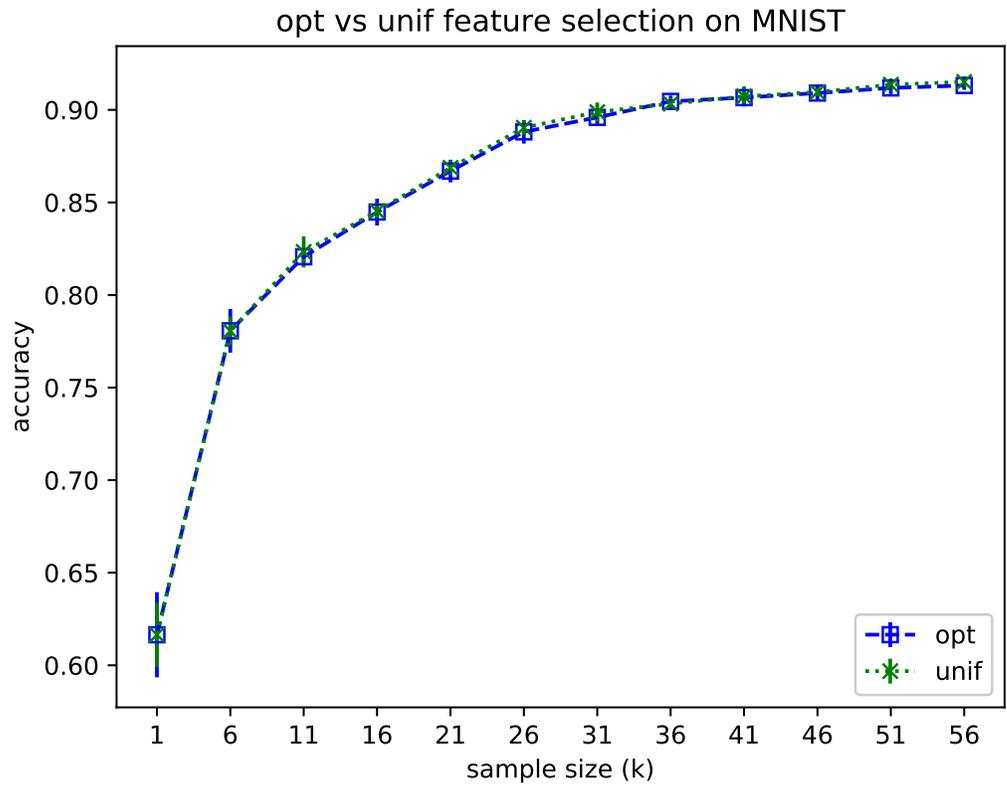}
  \caption{The classification accuracy of RFSVM with the simple random feature
  selection (``unif'') and the reweighted feature selection (``opt'')
  are shown for different sample sizes in the hand-written digit recognition
  (MNIST)}
\end{figure}

\end{document}